\documentclass[12pt,reqno]{amsart}
\usepackage[foot]{amsaddr}
\usepackage{amsmath,amsfonts,amsthm,amssymb,color}
\usepackage[usenames,dvipsnames]{xcolor}
\usepackage[T1]{fontenc}
\usepackage{enumerate}
\usepackage{hyperref}
\hypersetup{
     colorlinks   = true,
     citecolor    = blue,
     linkcolor    = blue
}
\usepackage{comment}
\usepackage{pdfsync}{\tiny }

\newcommand{\teal}{\color{black}}

\usepackage[left=1in, right=1in, top=1.1in,bottom=1.1in]{geometry}
\numberwithin{equation}{section}

\newcommand{\E}{\mathbb{E}}
\newcommand{\R}{\mathbb{R}}

\newcommand{\mO}{\mathcal O}

\newtheorem{thm}{Theorem}[section]
\newtheorem{example}[thm]{Example}

\newtheorem{lemma}[thm]{Lemma}

\newtheorem{prop}[thm]{Proposition}
\newtheorem{obs}[thm]{Remark}
\newtheorem{asspt}[thm]{Assumption}
\newtheorem{definition}[thm]{Definition}

\newcommand{\norm}[1]{\left|\left|#1\right|\right|}

\newcommand{\gives}{\rightarrow}
\newcommand{\lr}[1]{\left(#1\right)}

\newcommand{\bk}[1]{\left\langle #1\right\rangle}
\newcommand{\mF}{\mathcal F}
\newcommand{\Eof}[1]{\mathbb E\left[#1\right]}
\newcommand{\inprod}[2]{\left \langle #1,\,#2 \right\rangle}
\newcommand{\abs}[1]{\left|#1\right|}
\newcommand{\twomat}[4]{\lr{\begin{array}{cc} #1 & #2 \\ #3 & #4\end{array}}}
\newcommand{\set}[1]{\left\{#1\right\}}

\def\1{{\rm l}\hskip -0.21truecm 1}

\author{S. Favaro}
\address{stefano.favaro@unito.it \\Department of Economics and Statistics, University of Torino and Collegio Carlo Alberto}
\author{B. Hanin}
\address{bhanin@princeton.edu\\ Department of Operations Research and Financial Engineering\\ Princeton University}
\author{D. Marinucci}
\address{marinucc@mat.uniroma2.it \\Department of Mathematics, University of Rome Tor Vergata}
\author{I. Nourdin}
\address{ivan.nourdin@uni.lu\\ Department of Mathematics, Luxembourg University}
\author{G. Peccati}
\address{giovanni.peccati@uni.lu\\ Department of Mathematics, Luxembourg University}

\title{Quantitative CLTs in Deep Neural Networks}
\begin{document}

\begin{abstract}
    We study the distribution of a fully connected neural network with random Gaussian weights and biases in which the hidden layer widths are proportional to a large constant $n$. Under mild assumptions on the non-linearity, we obtain quantitative bounds on normal approximations valid at large but finite $n$ and any fixed network depth. Our theorems show  both for the finite-dimensional distributions and the entire process, that the distance between a random fully connected network (and its derivatives) to the corresponding infinite width Gaussian process  scales like $n^{-\gamma}$ for $\gamma>0$, with the exponent depending on the metric used to measure discrepancy. Our bounds are strictly stronger in terms of their dependence on network width than any previously available in the literature; in the one-dimensional case, we also prove that they are optimal, i.e., we establish matching lower bounds.\\

    \end{abstract}
\maketitle

\section{Introduction}\label{sec:intro}

{\bf Deep neural networks} \cite{lecun2015} are parameterized families of functions at the core of many recent advances in domains such as computer vision (e.g. self-driving cars \cite{krizhevsky2012imagenet}), natural language processing (e.g. ChatGPT \cite{brown2020language}), and structural biology (e.g. AlphaFold \cite{jumper2021highly}). The practical success of deep learning has led to significant interest in theoretical approaches to understanding how neural networks work and how to make them more efficient. 

As we are about to explain, an important chapter in deep learning theory seeks to understand the distribution of neural networks with randomly chosen parameters. This is the context for the present article, whose goal is to derive new quantitative CLTs for wide neural networks with random weights and biases. 

To motivate and informally introduce our results, recall that the typical use of neural networks in practice is to approximate an unknown function $f$ from a training dataset
\[
\set{\lr{x_\alpha,f(x_\alpha)} : \alpha=1,2,\ldots, k}
\]
consisting of its values at $k$ different inputs. Given the training data, one then fixes a \textbf{neural network architecture}, which specifies a parametric family of neural networks, and searches in this family for an approximation to $f$. In this article we will study the simplest, so-called fully connected, network architectures:
\begin{definition}[Fully Connected Network]\label{D:FC}
Fix a positive integer $L$ as well as $L+2$ positive integers $n_0,\ldots, n_{L+1}$ and a function $\sigma:\R\gives \R$. A {\bf fully connected depth $L$ neural network} with input dimension $n_0$, output dimension $n_{L+1}$, hidden layer widths $n_1,\ldots, n_L$, and non-linearity $\sigma$ is any function  $x_\alpha\in \R^{n_0}\mapsto z_\alpha^{(L+1)}\in \R^{n_{L+1}}$ of the following form
\begin{equation}\label{eq:z-def}
z_\alpha^{(\ell)} = \begin{cases}
W^{(1)}x_\alpha+b^{(1)},&\quad \ell=1\\
W^{(\ell)}\sigma(z_\alpha^{(\ell-1)})+b^{(\ell)},&\quad \ell=2,\ldots, L+1
\end{cases},  \qquad z_\alpha^{(\ell)}\in \R^{n_\ell},
\end{equation} 
where $W^{(\ell)}\in \R^{n_{\ell}\times n_{\ell-1}}$ are matrices, $b^{(\ell)}\in \R^{n_\ell}$ are vectors, and $\sigma$ applied to a vector is shorthand for $\sigma$ applied to each component. 
\end{definition}
The trainable parameters of a fully connected network are the \textbf{network weights} $W_{ij}^{_{(\ell)}}$ (entries of the weight matrices $W^{{(\ell)}}$) and \textbf{network biases} $b_i^{_{(\ell)}}$ (components of the bias vectors $b^{{(\ell)}}$). Of course, the network architecture and dataset must be compatible in the sense that $f$ must be a function from $\R^{n_0}$ to $\R^{n_{L+1}}$. For a training dataset and a network architecture, the goal is to find a setting of the weights and biases so that not only do we have
\[
z_{\alpha}^{(L+1)} \approx f(x_\alpha)
\]
for $x_\alpha$ in the training dataset but also for inputs not included in the training data. This optimization is typically done in two steps:\\
\begin{itemize}
    \item[(1)] Randomly initialize (i.e. sample) the network weights and biases.\\
    \item[(2)] Optimize the weights and biases by some variant of gradient descent on an empirical loss such as the squared error:
    \[
    \sum_{\alpha=1}^k \norm{z_\alpha^{(L+1)}-f(x_\alpha)}_2^2.
    \]
\end{itemize}
We thus see that neural networks with random weights and biases, the main subject of this article, describe the properties of neural networks at the start of training. The usual way to initialize parameters in practice leads to the following:
\begin{definition}[Random Fully Connected Neural Network]\label{def:rand-nn}
    Fix $L\geq 1,\, n_0,\ldots, n_{L+1}\geq 1,\, \sigma:\R\gives \R$ as well as $C_b\geq 0,\, C_W>0$. A {\bf random depth $L$ neural network} with input dimension $n_0$, output dimension $n_{L+1}$, hidden layer widths $n_1,\ldots, n_L$, and non-linearity $\sigma$, is the random field \eqref{eq:z-def} with random weights and biases:
    \begin{equation}\label{eq:Wb-def}
    W_{ij}^{(\ell)}\sim \mathcal N\lr{0,\frac{C_W}{n_{\ell-1}}},\qquad b_i^{(\ell)}\sim \mathcal N(0,C_b)\qquad \text{independent}.    
    \end{equation}
\end{definition}

In general, both describing the distribution of a randomly initialized neural network and tracking the dynamics of optimization is quite difficult. To make progress, several influential lines of research study these questions asymptotically when the network widths $n_1,\ldots, n_L$ are large \cite{neal1996,Der2006,Lee2018,Garriga2018, Matt2018,Yang2019,Bracale2021,hanin2021random,sirignano2020mean,sirignano2021mean, Jacot2018,bordelon2022self,mei2022generalization,du2018gradient,roberts2022principles,yaida2020non, ariosto2022statistical}. That neural networks simplify significantly in this \textit{infinite width limit} can already be seen at initialization: 
\begin{thm}[Infinite Networks as Gaussian Processes -- \cite{neal1996,Lee2018,Matt2018, Yang2020,Bracale2021,hanin2021random})]\label{thm:iwl}
Fix $L, n_0, n_{L+1},r\geq 1$ and a non-linearity $\sigma:\R\gives \R$ that is polynomially bounded to order $r$ in the sense of the forthcoming formula \eqref{eq:sigma-regg}. As $n_1,\ldots n_L\gives \infty$, the random field $x_\alpha\in \R^{n_0}\mapsto z_\alpha^{(L+1)}\in \R^{n_{L+1}}$ converges weakly in distribution, as an element of $C^{r-1}(\R^{n_0}, \R^{n_{L+1}})$, to a Gaussian process with $n_{L+1}$ iid centered components $(z_{i;\alpha}^{(L+1)},\, i=1,\ldots, n_{L+1})$ with limiting covariance
\[
K_{\alpha\beta}^{(L+1)}:=\lim_{n_1\ldots, n_L\gives \infty} \mathrm{Cov}\lr{z_{i;\alpha}^{(L+1)},z_{i;\beta}^{(L+1)}}
\]
satisfying
\begin{equation}\label{eq:K-rec}
   K_{\alpha\beta}^{(\ell+1)} = \begin{cases}
       C_b + C_W\bk{\sigma\lr{z_{i;\alpha}^{(\ell)}}\sigma\lr{z_{i;\beta}^{(\ell)}}}_{K^{(\ell)}},&\quad \ell \geq 1\\
       C_b + \frac{C_W}{n_0}x_\alpha \cdot x_\beta,&\quad \ell=0
   \end{cases},
\end{equation}
where for any $f:\R^2\gives \R$ we've written  $\bk{f(z_{i;\alpha}^{(\ell)},z_{i;\beta}^{(\ell)})}_{K^{(\ell)}}$ for the average values of $f$ with respect to the distribution
\[
\lr{z_{i;\alpha}^{(\ell)},z_{i;\beta}^{(\ell)}} \sim \mathcal N\lr{0,\twomat{K_{\alpha\alpha}^{(\ell)}}{K_{\alpha\beta}^{(\ell)}}{K_{\alpha\beta}^{(\ell)}}{K_{\beta\beta}^{(\ell)}}}.
\]
\end{thm}
\noindent We can now state the main question taken up in this article:
\begin{align}
    \label{eq:main-q}\text{\textbf{Question:} }&\text{How close is a random neural network at finite width to the}\\
    \notag &\text{infinite width Gaussian process described in Theorem \ref{thm:iwl}?}
\end{align}

Before discussing our results, we briefly discuss three motivations for studying random neural networks:
\begin{itemize}
    \item An important motivation for taking up  question \eqref{eq:main-q} comes from prior work on the neural tangent kernel (NTK) regime \cite{Jacot2018,du2018gradient,Arora2019,Lee2019,Yang2019,Yang2020,Yang2021, roberts2022principles, hanin2018neural,yaida2020non}, which occurs when $L,n_0,n_{L+1}$, and the training dataset are fixed,  weights and biases are initialized as in \eqref{eq:Wb-def}, and the hidden layer weights $n_1,\ldots, n_L$ are sent to infinity. The NTK regime has two salient features: \\
\begin{itemize}
    \item The stochastic processes $x_\alpha\mapsto z_\alpha^{(L+1)}$ converge in distribution to centered Gaussian processes with independent components. We already saw this in Theorem \ref{thm:iwl}.\\
    \item Using sufficiently small learning rates and losses such as the mean squared error, the entire trajectory of optimization coincides with the one obtained by replacing the non-linear network $z_\alpha^{(L+1)}$ by its linearization around the randomly initialized setting of network weights and biases (see \cite{Jacot2018}, Theorem 3.2 in \cite{Arora2019}, and Theorem 5.4 in \cite{bartlett2021deep}).\\
\end{itemize}
The second point shows that in the infinite width limit optimization will not get stuck in a spurious local minimum of the training loss, as the loss is convex after we replace the network by its linearization. More precisely, the network function $x_\alpha\mapsto z_\alpha^{(L+1)}$ can be replaced by its linearization \textit{at the start of training}:
        \[
            x_\alpha\mapsto z_\alpha^{(L+1)}(\theta_0) + \inprod{\nabla_\theta z_\alpha^{(L+1)}(\theta_0)}{\theta- \theta_0},
        \]
        where $\theta_0$ denotes the parameters at initialization. In this regime, it is the NTK, i.e. the deterministic limiting kernel
        \[
            \Theta_{\alpha\beta}^{(L+1)}:=\lim_{n_1,\ldots, n_L\gives \infty} \inprod{\nabla_\theta z_\alpha(\theta_0)}{\nabla_\theta z_\beta(\theta_0)},\qquad x_\alpha, x_\beta\in \R^{n_0}
        \]
        that controls both the dynamics of training and the out-of-sample performance of the network. Note that this kernel depends only on the neural network at initialization. Hence, in this setting virtually all questions about network performance are determined at initialization.  

Taking the width to infinity in the NTK regime comes at a steep explanatory cost. Indeed, one of the most important practical features of neural networks is precisely that they are \textit{not} linear models and hence learn data-dependent features \cite{Chizat2019}. The NTK regime is thus too rigid to capture important aspects of the behavior of realistic neural networks. To study non-linear effects such as feature learning one must either change the initialization scheme (leading to the mean-field limit \cite{Mei2018,Chizat2018,Rotskoff2018,sirignano2020mean}), consider regimes in which the training dataset size grows with network width (see e.g. \cite{Dou2021,naveh2021self,cui2023optimal,seroussi2023separation,hanin2023bayesian, ariosto2022statistical}), or study neural networks at finite width (see e.g. \cite{Noci2021,Zavatone2021,hanin2023bayesian,roberts2022principles,yaida2020non,hanin2018neural,hanin2019finite}). In this article, we focus on this last option and develop new probabilistic tools for analyzing neural networks at finite but large width (see \eqref{eq:large-width}). \\

\item Regardless of whether one is studying neural networks in the NTK regime or not, the analysis at initialization gives information about \textit{early training.} This comes in several flavors:\\
        \begin{itemize}
            \item Studying how to initialize so that the forward and backward pass at are numerically stable at initialization (i.e. have finite mean and variance) \cite{hanin2018neural,hanin2018start,zhang2019fixup,DBLP:conf/aistats/PenningtonSG18,pennington2017resurrecting,schoenholz2016deepinformation,he2015delving,dezoort2023principles}\\
            \item Studying the first step of optimization to see its effect on feature learning \cite{ba2022high,danditwo,hanin2019finite}\\
        \end{itemize}
        
        \item Understanding the distribution of randomly initialized neural networks is crucial when studying Bayesian inference with neural networks. At least in regression tasks (i.e. those with quadratic log-likelihoods), the key is to characterize the Laplace transform of the prior distribution over the field $x_\alpha\mapsto z_\alpha^{(L+1)}$ \cite{yaida2020non,hanin2023bayesian, roberts2022principles,li2021statistical, naveh2021self,seroussi2023separation,fischer2024critical,ariosto2022statistical,cui2023optimal}.\\
    \end{itemize}


\subsection{Informal Overview of Results}
Our main results, which we present in more detail in \S \ref{sec:res} below, can be summarized as follows:
\begin{enumerate}
    \item \textbf{One-dimensional Quantitative CLTs (QCLTs).} For a fixed network input $x_\alpha\in \R^{n_0}$ we consider a single component $z_{i;\alpha}^{(L+1)}$ of the network output. Here we ask, as a function of network width, how close is the distribution of $z_{i;\alpha}^{(L+1)}$ (and its derivatives with respect to $x_\alpha$) to the corresponding infinite width Gaussian? We find that the total-variation distance between them is bounded above by a constant times $(\text{width})^{-1}$; we also prove that this rate is optimal  as soon as $L\geq 1$, i.e., we establish matching lower bounds (in contrast, when $L=0$ outputs $z_{i;\alpha}^{(1)}$ are exactly Gaussian for any width). See Theorem \ref{thm:one-d} for the precise statement. \\
    \item \textbf{Finite-dimensional QCLTs.} For a fixed finite collection of network inputs $x_\alpha\in \R^{n_0},\, \alpha \in \mathcal A$, we obtain in Theorem \ref{thm:finite-d} upper bounds on the convex distance (see \eqref{e:dconv}) between the vector $\lr{z_{i;\alpha}^{(L+1)},\, \alpha\in \mathcal A}$ (and its derivatives with respect to $x_\alpha$) and the corresponding Gaussian. Here we find an upper bound of the order of $(\text{width})^{-1/2}$, with a pre-factor that scales polynomially with the number of elements in $\mathcal A$. We conjecture that this rate is sub-optimal and can be improved to be a constant depending on $\mathcal A$ times $(\text{width})^{-1}$.\\
    \item \textbf{Functional QCLTs.} We prove upper bounds in Theorem \ref{thm:infinite-d} between $z_{\alpha}^{(L+1)}$ -- viewed as an element of the infinite-dimensional Sobolev space of weakly differentiable functions from a compact set in $\R^{n_0}$ to $\R^{n_{L+1}}$ -- and its infinite width limit. These bounds are in terms of both the Wasserstein-2 metric and the $d_2$ distance (see \eqref{eq:d2-def}) and scale like ${\rm (width)}^{-\kappa}$ for certain exponents $\kappa \in (0,1/2]$. {{}When $\sigma$ is either a smooth or the ReLU non-linearity, and we study the Wasserstein-2 metric in some appropriate Sobolev space, we can take $\kappa = 1/8$. In Theorem \ref{t:sup} -- which is one of the most innovative contributions of our work and applies to smooth nonlinearities -- we use a Sobolev embedding argument to deduce upper bounds on transport distances associated with supremum norms on spaces of differentiable functions. In this case, we again achieve bounds that scale as ${\rm (width)}^{-1/8}$. }  \\
\end{enumerate}
As we review in more detail in \S \ref{sec:lit-rev}, our results in cases (1)---(3) above are strictly stronger in terms of their dependence on network width than any previously available in the literature. We conclude by emphasizing that, on a technical level, this article introduces several novel ideas:\\
\begin{itemize}
\item In the one-dimensional case, we prove and use a new optimal bound (stated in Proposition \ref{p:ibp}) on the total variation and 1-Wasserstein distances between the law of a conditionally Gaussian random variable and of a Gaussian random variable. This result -- that is of independent interest -- extends techniques and ideas initially introduced in \cite{nourdin2009chaos, nourdin2012normal, nourdin2015optimal}.  \\

    \item We extend ideas from the work of \cite{basteri2022quantitative} to obtain quantitative CLTs with respect to the convex distance (see \eqref{e:dconv}) in the setting of possibly degenerate covariances (see Proposition \ref{p:ibp3}). This is used in deriving the finite-dimensional CLTs from Theorem \ref{thm:finite-d}. In general, our way of deducing one- and finite-dimensional approximations revolves around the use of integration by parts formulae and the so-called {\bf Stein's method} for probabilistic approximations -- see \cite{ledoux2015logsob, nourdin2012normal, nourdin2022multivariate}.\\
    \item We formulate, based on ideas from \cite{dierickx2023small}, a novel coupling argument for conditionally Gaussian fields (see Proposition \ref{p:simple}) that is used in the proof of Theorem \ref{thm:infinite-d} and Theorem \ref{t:sup}. \\
    \item {{}We develop a new application of {\bf modified Powers-St\o rmer inequalities} \cite{powers1970bounds}, that we formally state (in full generality) in Proposition \ref{p:bello} below. Such a result, that extends bounds already established in \cite{dierickx2023small}, allows one to upper bound the Hilbert-Schmidt norm of the difference of the square roots of two positive semi-definite operators without requiring that one of them is strictly positive definite.} This will be used in the proof of Theorem \ref{thm:infinite-d}. Our approach should be compared with the discussion contained in \cite[Section 5]{basteri2022quantitative}, where some alternate strategies for deriving functional bounds is partially outlined.
\end{itemize}

\begin{obs}{\rm The recent (independently written) paper \cite{torrisi23} uses Stein's method in a way comparable to ours to deduce upper bounds in the total variation and convex distances for shallow and deep networks, in the case where the input is fixed and only the network's output is concerned (no derivatives). We stress that our paper focuses on the derivation of tight probabilistic bounds in terms of the width of the neural network, providing only a partial description of the analytical dependence of the constants on the other parameters of the model (e.g. $C_W, C_b$). While it is in principle possible to explicit such dependence in a finite-dimensional setting (see e.g. \cite{basteri2022quantitative, Bordino2023, torrisi23}) the task becomes much more complicated in a functional framework since in this case the constants in our bounds crucially depend on certain traces of integral operators that one cannot directly represent in terms of the involved parameters. 

We prefer to think of this point as a separate issue, and leave it open for further research.
}
\end{obs}

\subsection{Outline for Remainder of Article}
The rest of this article is structured as follows. First, in \S \ref{sec:setting}, we formally introduce assumptions and notation related to the non-linearity $\sigma$. Then, in \S \ref{sec:res}, we state our precise results on one-dimensional (\S \ref{s:1d}), finite-dimensional (\S \ref{s:fd}), and infinite-dimensional quantitative CLTs (\S \ref{s:functional}). Throughout \S \ref{sec:res} we compare our results to prior work and mention a range of further related articles in \S \ref{sec:lit-rev}. We then develop in \S \ref{s:prep}  some preparatory results that will be used in the proofs of our main Theorems. Specifically, \S \ref{ss:prepcum} builds on the simple observation from Lemma \ref{lem:cond-gauss} that random neural networks are conditionally Gaussian and recalls key estimates on the fluctuations of the conditional covariances (see Theorem \ref{thm:var-of-var}). Next, \S \ref{ss:preps1} recalls Stein's method for one-dimensional quantitative CLTs, while \S \ref{ss:preps2} and \S \ref{ss:prepfunc} provide the finite-dimensional and infinite-dimensional extensions, respectively. Several of these extensions (specifically Propositions \ref{p:ibp3}, \ref{p:bello}, and \ref{p:simple}) are elementary but new. Finally, in \S \ref{sec:pfs} we complete the proofs of our main results. 


\section{Assumptions and Definitions}\label{sec:setting}
For our precise results, we will need the following mild technical condition on the activation function.

\begin{definition}[Polynomially Bounded Activations]\label{D:poly-bdd}
For fixed $r\geq 1$, we say that the non-linearity $\sigma:\R\gives \R$ is {\bf polynomially bounded to order $r$}, if either $\sigma$ is $r$ times continuously differentiable or if it {{}is} $r-1$ times continuously differentiable and its $(r-1)-$st derivative is a continuous piecewise linear function with a finite number of points of discontinuity for its derivative. In either case we also require that the $r$-th derivative of $\sigma$ is polynomially bounded
\begin{equation}\label{eq:sigma-regg}
\exists k \geq 0\text{ s.t. }\norm{(1+\abs{x})^{-k}\frac{d^{r}}{dx^{r}}\sigma(x)}_{L^\infty(\R)}<\infty,
\end{equation}
and that for every fixed $x_\alpha, x_\beta$ and $I,J$ such that $|I|, |J|=r$, the mixed derivatives $D^J_\alpha D_\beta^I \Sigma_{\alpha\beta}^{(\ell)}$ are well-defined and finite with probability one, where $\Sigma^{(\ell)}$ is defined according to the forthcoming formula \eqref{eq:Sig-def}.
\end{definition}

\medskip
\begin{obs}{\rm 
The condition that the mixed derivatives  $D^J_\alpha D_\beta^I \Sigma_{\alpha\beta}^{(\ell)}$ are well-defined and finite with probability one hold in the following settings, which cover virtually all cases of practical importance:
\begin{itemize}
    \item $r$ is arbitrary and $\sigma$ is a smooth function. 
    \item $r$ is arbitrary and  $\sigma$ is strictly monotone (e.g. leaky ReLU with $r=1$)
    \item $r$ is arbitrary and the bias variance $C_b$ (see \eqref{eq:Wb-def}) is strictly positive. 
    \item $r=1, C_b=0$ and $\sigma=\mathrm{ReLU}$ and we restrict the network inputs $x_\alpha,x_\beta$ to be non-zero (this case is somewhat more subtle and is proved in the course of establishing Proposition 9 in \cite{hanin2019complexity}).
\end{itemize}

}
\end{obs}

\medskip
Virtually all non-linearities used in practice are polynomially bounded to some order $r\geq 1$. This is true, for instance, of smooth non-linearities such as $\tanh$ and $\mathrm{GeLU}$ \cite{hendrycks2016gaussian} (in which case $r$ is arbitrary) as well as piecewise linear non-linearities such $\mathrm{ReLU}$  \cite{krizhevsky2012imagenet} and leaky $\mathrm{ReLU}$ \cite{he2015delving}.

Since we are concerned in this article with quantitative central limit theorems not only for the outputs of a random neural network but also for its derivatives with respect to network inputs, let us agree that for a {\bf multi-index} $J=\lr{j_{1},\ldots, j_{n_0}}\in \mathbb N^{n_0}$ we will write  $\abs{J}:=j_{1}+\cdots +j_{n_0}$ for the {\bf order} of $J$ and let
\begin{equation}\label{e:differentials}
D_{\alpha}^{J} =\partial_{x_1}^{j_{1}}\cdots \partial_{x_{n_0}}^{j_{n_0}}\bigg|_{x=x_{\alpha} = (x_1,...,x_{n_0})}
\end{equation}
denote the corresponding partial derivative operator, with the convention that $D^{0} = {\rm Id.}$

\begin{obs} \label{r:peaceofmind}{\rm For fixed $\ell = 1,...,L+1$, consider the real valued random field $x_\alpha\in \R^{n_0}\mapsto z_{i;\alpha}^{(\ell)}\in \R$ defined in \eqref{eq:z-def}, and denote by $\Gamma$ the centered Gaussian field having covariance $K^{(\ell)}$, as defined in \eqref{eq:K-rec}. Assume that $\sigma$ is polynomially bounded to the order $r$ (Definition \ref{D:poly-bdd}), and that weights and biases are selected according to Definition \ref{def:rand-nn}. Then, one can verify the following properties:
\begin{enumerate}
\item[(i)] both $\Gamma$ and $z_i^{(\ell)}$ are of class $C^{r-1}$ with probability one;
\item[(ii)] with probability one, both $z_i^{(\ell)}$ and $\Gamma$ are Lebesgue-almost everywhere $r$ times differentiable and, for all $J$ such that $|J|=r$, there exist versions of the fields $x_\alpha \mapsto D^J_\alpha z_{i;\alpha}^{(\ell)}  $ and $x_\alpha \mapsto D^J_\alpha \Gamma_\alpha$ that are locally bounded;
\item[(iii)] for every fixed non-zero $x_\alpha$ and $J$ such that $|J|=r$, the mixed derivatives $D^J_\alpha z_{i;\alpha}^{(\ell)} $ and $D^J_\alpha \Gamma_\alpha$ are well-defined and finite with probability one;
\item[(iv)] For every $x_\alpha, x_\beta \in \R^{n_0}$ and every $I,J$ such that $|I|,|J|\leq r-1$, 
\begin{equation}\label{e:derividentitiy}
\mathbb{E}[D_\alpha^I\Gamma_\alpha \cdot D_\beta^J \Gamma_\beta] = D_\alpha^I D_\beta^J K^{(\ell)}_{\alpha\beta},
\end{equation}
and the mapping $(x_\alpha, x_\beta)\mapsto D_\alpha^I D_\beta^J K^{(\ell)}_{\alpha\beta} $ is continuous. {{}Relation \eqref{e:derividentitiy} continues to hold whenever $\max\{ |J|,|I|\} =r$ and one considers non-zero inputs $x_\alpha, x_\beta$, and in this case there exists a version of the mapping $(x_\alpha, x_\beta)\mapsto D_\alpha^I D_\beta^J K^{(\ell)}_{\alpha\beta} $ that is bounded on compact sets. }
\end{enumerate}
}
\end{obs}

The following definition formalizes what it means for a random neural network, together with some of its derivatives, to have a non-degenerate covariance structure in the infinite width limit. This condition will be used as a hypothesis in most of our results.

\begin{definition}\label{d:nondeg} Fix $r\geq 1$ and suppose that $\sigma$ is polynomially bounded to order $r$, in the sense of Definition \ref{D:poly-bdd}. Consider any finite set $\mathcal A$ indexing distinct network inputs
\[
x_{\mathcal A} = \{x_\alpha : \alpha \in \mathcal A\}\subseteq \R^{n_0}
\]
and any finite set of directional derivative operators 
\begin{equation}\label{e:V}
V=\set{V_{1},\ldots, V_{p}},\quad V_{j}:=\sum_{i=1}^{n_0} v_{ij}\partial_{x_i}.
\end{equation}
We say that the infinite-width covariance structure $\{K^{(\ell)} : \ell = 1,...,L+1\}$ from Theorem \ref{thm:iwl} is {\bf non-degenerate on $x_{\mathcal A}$ to order $q\leq r$ with respect to $V$} if, for all $\ell = 1,..., L+1$, the infinite width covariance matrix 
\begin{equation}\label{e:kmin}
K_{\mathcal A,V}^{(\ell),\leq q}:=\lr{ V_{\alpha_1}^{J_1} V_{\alpha_2}^{J_2}K_{\alpha_1\alpha_2}^{(\ell)},\, |J_1|, |J_2|\leq q,\,\, \alpha_1,\alpha_2\in \mathcal A}
\end{equation}
is invertible, where for each multi-index $J_i=\lr{j_{i1},\ldots, j_{ip}}\in \mathbb N^{p}$ with order $|J_i|= j_{i1}+\cdots+j_{ip}$ we have written
\begin{equation}\label{e:v-differentials}
V_{\alpha_i}^{J_i} :=V_{1}^{j_{i1}}\cdots V_{p}^{j_{ip}}\bigg|_{x=x_{\alpha_i}}
\end{equation}
for the corresponding differential operators and have set $V^{0} = {\rm Id.}$ We stress that, in \eqref{e:kmin}, the rows and columns of the matrix $K_{\mathcal A,V}^{(\ell),\leq r}$ are indexed, respectively, by pairs $(J_1, \alpha_1)$ and $(J_2, \alpha_2)$. When the set $V$ is such that $p=n_0$ and $v_{i,j} = \delta_{ij}$, then $V^J = D^J$ --- as defined in \eqref{e:differentials}; in this special case, we will say that the infinite-width covariance structure $\{K^{(\ell)} : \ell = 1,...,L+1\}$ is {\bf canonically non-degenerate to the order} $q\leq r$ and use the notation $K_{\mathcal A,V}^{(\ell),\leq q} = K_{\mathcal A}^{(\ell),\leq q}$. 

\smallskip

\end{definition}

\smallskip

By virtue of \eqref{e:derividentitiy}, the invertibility of the covariance matrix $K_{\mathcal A,V}^{(\ell),\leq q}$ implies (thanks to Cauchy-Schwarz) that $K_{\mathcal A,V}^{(\ell),\leq q}$ has strictly positive diagonal terms. Plainly, the covariance structure $\{K^{(\ell)}\}$ is non-degenerate on $x_\mathcal{A}$ to order 0 if and only if the matrices $\{ K_{\alpha\beta}^{(\ell)} : \alpha, \beta \in \mathcal{A}\}$ defined in \eqref{eq:K-rec} are invertible for $\ell = 1,..., L+1$. In particular, if $\mathcal{A} = \{\alpha\}$ is a singleton, non-degeneracy to the order 0 simply means that $K^{(\ell)}_{\alpha\alpha}>0$ for $\ell = 1,..., L+1$. Finally, for $\ell=1,...,L+1$, we introduce the notation 
\begin{equation}   \label{e:kappa}
\kappa_{\alpha\beta}^{(\ell)}:=\mathrm{Cov}\lr{z_{i;\alpha}^{(\ell)},\, z_{i;\alpha}^{(\ell)}},
\end{equation}
and write
\[
\mathcal F^{(\ell)} :=\text{ sigma field generated by weights and biases in layers }1,\ldots, \ell. 
\]
Note that $\kappa_{\alpha\beta}^{(\ell)}\gives K_{\alpha\beta}^{(\ell)}$ as $n\gives \infty.$ We conclude this section with the following elementary lemma that will be used throughout the paper.

\begin{lemma}\label{lem:cond-gauss}
Conditionally on $\mF^{(\ell)}$ the random field $x_\alpha\in \R^{n_0} \mapsto z_{\alpha}^{(\ell+1)}\in \R^{n_{\ell+1}}$ has i.i.d. centered Gaussian components with conditional covariance
\begin{align*}
    \mathrm{Cov}\lr{z_{i;\alpha}^{(\ell+1)},\, z_{j;\beta}^{(\ell+1)}~|~\mathcal F^{(\ell)}} =\delta_{ij}\Sigma_{\alpha\beta}^{(\ell)}
\end{align*}
where
\begin{equation}\label{eq:Sig-def}
    \Sigma_{\alpha\beta}^{(\ell)}:=C_b + \frac{C_W}{n_\ell}\sum_{j=1}^{n_\ell}\sigma\lr{z_{j;\alpha}^{(\ell)}}\sigma\lr{z_{j;\beta}^{(\ell)}}.
\end{equation}
One has in particular that $\kappa_{\alpha\beta}^{(\ell)} = \mathbb{E}[\Sigma_{\alpha\beta}^{(\ell)}]$, where we used the notation \eqref{e:kappa}.
\end{lemma}

\begin{obs}\label{r:peaceofmind2}{\rm Assume that $\sigma$ is polynomially bounded to the order $r\geq 1$. In resonance with Remark \ref{r:peaceofmind}, the random covariance function $(x_\alpha, x_\beta)\mapsto \Sigma_{\alpha\beta}^{(\ell)} $ verifies the following properties:
\begin{enumerate}
\item[(i)] $\Sigma^{(\ell)} $ is of class $C^{r-1, r-1}(\mathbb{R}^{n_0}\times \mathbb{R}^{n_0})$ with probability one;
\item[(ii)] with probability one, $\Sigma^{(\ell)}$ is Lebesgue-almost everywhere $r$ times differentiable in each variable and, for all multi-indices $I,J$ such that $|I|, |J|=r$, there exist a version of the field $(x_\alpha, x_\beta)\mapsto {{}D^I_\alpha D^J_\beta \Sigma_{\alpha\beta}^{(\ell)} }$ that is locally bounded;
\item[(iii)] For every $x_\alpha, x_\beta \in \R^{n_0}$ and every $I,J$ such that $|I|,|J|\leq r$, 
\begin{equation}\label{e:derividentitiy2}
\mathbb{E}[D^J_\alpha D_\beta^I \Sigma_{\alpha\beta}^{(\ell)}] = D_\alpha^J D_\beta^I \kappa^{(\ell)}_{\alpha\beta},
\end{equation}
and the mapping $(x_\alpha, x_\beta)\mapsto \mathbb{E}[ ( D^J_\alpha D_\beta^I \Sigma_{\alpha\beta}^{(\ell)} )^2]$ is integrable over arbitrary compact sets.
\end{enumerate}

\medskip
}
\end{obs}

 \section{Main Results}\label{sec:res}
In the subsequent sections  we will present our main results, respectively, in the one-dimensional (\S \ref{s:1d}), the finite-dimensional (\S \ref{s:fd}), and the functional setting (\S \ref{s:functional}). Before stating them we present some notation in \S \ref{sec:notation}.

\subsection{Notation and Setting for Main Results }\label{sec:notation}
Our results give quantitative CLTs for random neural networks verifying the following (parameterized) set of assumptions. 

\begin{asspt}\label{a:cadre} {\rm Fix constants $c_1\geq c_2>0$, integers $r,L,n_0, n_{L+1}\geq 1$, scalars $C_b\geq 0,\, C_W>0$, and a mapping $\sigma:\R\gives \R$ that is polynomially bounded to order $r$ as in Definition \ref{D:poly-bdd}. We then consider a random neural network $x_{\alpha}\in \R^{n_0}\mapsto z_\alpha^{(L+1)}\in \R^{n_{L+1}}$ with input dimension $n_0$, output dimension $n_{L+1}$, hidden layer widths $n_1,\ldots, n_L$, and non-linearity $\sigma$ as in Definition \ref{def:rand-nn} and suppose that for some $n\geq 1$ 
\begin{equation}\label{eq:large-width}
    c_2 n\leq n_1,\ldots, n_L\leq c_1 n.
\end{equation}
 We make this proportional width assumption mainly because it is common in the literature. Our results below do not require it in the sense that we may simply replace $n$ by the minimum of $n_1,\ldots, n_L$. For the sake of brevity, we define the set of parameters 
\begin{equation}\label{e:P}
\mathcal{P}:= \{ c_1, c_2, L, n_0, C_b, C_W\}
\end{equation}
(note that $\mathcal{P}$ does not contain $r$).}
\end{asspt}

The results in this article give quantitative CLTs showing that when $n$ is large the random field $z_\alpha^{(L+1)}$ and its derivatives 
\begin{equation}\label{eq:deriv-def}
D_{\alpha}^{J}z_{i;\alpha}^{(L+1)}.  :=\partial_{x_1}^{j_{1}}\cdots \partial_{x_{n_0}}^{j_{n_0}}\bigg|_{x=x_{\alpha}}z_{i;\alpha}^{(L+1)},\qquad J = \lr{j_1,\ldots, j_{n_0}}\in \mathbb N^{n_0}   
\end{equation}
(or, more generally, the mixed directional derivatives appearing in \eqref{e:differentials}) are close to those of a centered Gaussian process with $n_{L+1}$ independent and identically distributed components.

Although the classes of probabilistic distances we study vary somewhat between the one-dimensional, finite-dimensional, and functional cases below, they all contain Wasserstein distances, which we recall here for the sake of readability. See \cite[Chapter 6]{villani2009transport} for further details and proofs.

\begin{definition}\label{d:wass} Let $K$ be a real separable Hilbert space, let $X,Y$ be two $K$-valued random elements, and fix $p\geq 1$. We define the $p$-{\bf Wasserstein distance}, between the distributions of $X$ and $Y$, to be the quantity
\begin{equation}\label{e:wassdef}
{\bf W}_p(X,Y) :=\left( \inf_{(T,S)} \mathbb{E}[\| T-S\|^p_K]\right)^{1/p},
\end{equation}
where the infimum runs over all random elements $(T,S)$ such that $T\stackrel{law}{=}X$ and  $S\stackrel{law}{=}Y$. 
\end{definition}

Definition \ref{d:wass} will be applied to $K=\mathbb{R}$ in Section \ref{s:1d}, to $K = \mathbb{R}^m$ in Section \ref{s:fd} and to $K$ equal to some appropriate Sobolev space in Section \ref{s:functional}\footnote{For the rest of the paper, we will use the notation $\mathcal{B}(K)$ to indicate the class of all Borel subsets of $K$.}. We note that, trivially, ${\bf W}_p\leq {\bf W}_q$ for $p\leq q$, and also record the following two additional facts:
\begin{enumerate}
    \item[--] {if $U$ is an arbitrary random element defined on the same probability space as $X$ (taking values in some Polish space $\mathcal{U}$ and with law $\mathbb{P}_U$), then there exists a version 
    $$
 \mathbb{P}_{X|U} :   \mathcal{U}\times \mathcal{B(K)}\to [0,1] : (u,B)\mapsto \mathbb{P}_{X|U = u}(B)
    $$
of the conditional distribution of $X$ given $U$ such that 
\begin{eqnarray}\label{e:w2cond}
{\bf W}^q_q(X,Y) \leq \int_\mathcal{U} {\bf W}_q^q ( \mathbb{P}_{X|U=u}, \mathbb{P}_{Y})\, d\mathbb{P}_U(u),
\end{eqnarray}
see e.g. \cite{basterimasterthesis};
}

\item[--] in the case $q=1$ one has the dual representation
\begin{equation}\label{e:w1dual}
{\bf W}_1(X,Y) = \sup_{h\in {\rm Lip}(1)} | \mathbb{E}h(X) - \mathbb{E}h(Y) |,
\end{equation}
where the supremum runs over all $1$-Lipschitz mappings on $K$, that is, all real-valued mappings $h$ such that $|h(a) - h(b)|\leq \|a-b\|_K$, for all $a,b\in K$.
\end{enumerate}

\subsection{One-dimensional bounds}\label{s:1d} Our first result, Theorem \ref{thm:one-d}, measures the total variation and 1-Wasserstein distances between the output of a random neural network evaluated at a single input and a Gaussian random variable. To state it, recall that, given random variables $X,Y$, the {\bf total variation distance} between the distributions of $X$ and $Y$ is defined as
\begin{equation}\label{e:tvd}
d_{TV}(X,Y) := \sup_{B\in \mathcal{B}(\mathbb{R})} |\mathbb{P}(X\in B) - \mathbb{P}(Y\in B) |
\end{equation}
where $\mathcal B(\R)$ denotes the Borel-measurable subsets of $\R$.

{\begin{thm}\label{thm:one-d}
Consider a random neural network $x_\alpha\in \R^{n_0}\mapsto z_\alpha^{(L+1)}\in \R^{n_{L+1}}$ verifying Assumption \ref{a:cadre} with a non-linearity $\sigma$ that is polynomially bounded to order $r\geq 1$ as in Definition \ref{D:poly-bdd}, and recall notation \eqref{e:P}. Fix a network input $x_\alpha\in \R^{n_0}$, and directional derivative operators $V = \{V_1,...,V_p\}$ like in \eqref{e:V}. Fix also a multi-index $J\in \mathbb N^p$ such that $|J|\leq r$, and let $Z$ be a centered Gaussian random variable with variance $V^J_\alpha V^J_\beta K_{\alpha\beta}^{(L+1)}|_{x_\alpha = x_\beta}$, where we have adopted the notation \eqref{e:differentials}. If the infinite-width covariance structure $\{K^{(\ell)}\}$ is non-degenerate on the singleton $\set{x_\alpha}$ to order $q= |J|\leq r$ with respect to $V$, in the sense of Definition \ref{d:nondeg}, then the following conclusions hold:
\begin{enumerate}
\item[{\bf (1)}] there exists $C>0$, depending on $\sigma,r,V,J,x_\alpha, \mathcal{P}$, with the property that, for each $i=1,\ldots n_{L+1}$,
    \begin{equation}\label{eq:one-d-tv-1}
        \max\Big\{ {\bf W}_1(V^J_\alpha z_{i;\alpha}^{(L+1)}, Z),\,\, d_{TV}(V^J_\alpha z_{i;\alpha}^{(L+1)}, Z) \Big\}\leq C n^{-1},
    \end{equation}
    and the constant $C$ can be chosen  uniformly when $\norm{x_\alpha}^2/n_0$ varies over a compact set;
 \item[{\bf (2)}] the dependence on $n$ in \eqref{eq:one-d-tv-1} is optimal when $q=0$, in the following sense: denoting by $Z'$ a centered Gaussian random variable with the same variance as $z_{i;\alpha}^{(L+1)}$, there exists $C_0>0$, depending on $x_\alpha$ and $\mathcal{P}$, such that, for each $i=1,\ldots n_{L+1}$,
 \begin{equation}\label{eq:lowerb}
        \min\Big\{ {\bf W}_1( z_{i;\alpha}^{(L+1)}, Z'),\,\, d_{TV}( z_{i;\alpha}^{(L+1)}, Z') \Big\}\geq C_0 n^{-1}.
    \end{equation}
    \end{enumerate}
\end{thm}
}
We prove Theorem \ref{thm:one-d} in \S \ref{sec:one-d}. Before presenting our results in the multi-dimensional setting, we make several remarks:

\begin{itemize}
\item[(a)] Let us give two examples of situations where Theorem \ref{thm:one-d} applies:\\
\begin{itemize}
    \item When $\sigma(t)=\mathrm{ReLU}(t)=\max\set{0,t}$, we may take $C_b=0,C_W=2$ and
    \[
     V= \set{\partial_{x_i}}~\text{for some }i.
    \]
    For any non-zero network input $x_\alpha$, a simple computation shows that
    \[
    K_{\alpha\alpha}^{(\ell)} = \frac{2}{n_0}\norm{x_\alpha}^2,\qquad \partial_{x_{i;\alpha}}\partial_{x_{i;\beta}} K_{\alpha\beta}^{(\ell)}\big|_{x_{\alpha}=x_\beta}=\frac{2}{n_0}.
    \]
    Hence, the infinite width covariance structure is non-degenerate on the singleton $\set{x_\alpha}$ both to order $0$ and to order $1$ with respect to $V$. \\
    \item By inspection of the proof of Theorem \ref{thm:one-d} (given in Section \ref{sec:one-d}) and by virtue of Theorem \ref{thm:var-of-var}, one sees that the conclusion continues to hold if $\sigma$ is smooth (that is, of class $C^\infty(\mathbb{R})$) and 
$V^J_\alpha V^J_\beta K_{\alpha\beta}^{(L+1)}|_{x_\alpha=x_\beta}>0.$ In particular, we may take $\sigma$ to be any smooth function such as $\tanh(t)$ and set $C_b=0,C_W=1$. For any non-zero network input $x_\alpha$ the recursion from Theorem \ref{thm:iwl} then reads
\[
K_{\alpha\alpha}^{(\ell+1)} = \bk{\sigma(z_{i;\alpha}^{(\ell)})^2}_{K^{(\ell)}},
\]
showing that the infinite width covariance structure is non-degenerate on the singleton $\set{x_\alpha}$ to order $q=0$.\\
\end{itemize}
\item[\rm (b)] We recall that $V^{0}_\alpha$ corresponds to the identity operator, so that Theorem \ref{thm:one-d} in the case $|J|=0$ yields quantitative CLTs for the random variables $z^{(L+1)}_{i;\alpha}$.
    \item[\rm (c)] When $|J| = 0$ and $L=1$, our estimates on the total variation distance strictly improve those proved in \cite[Theorem 4]{Bordino2023} and \cite[Theorem 4.1]{torrisi23}, that obtain a rate of convergence of the order $n^{-1/2}$ by using some version of Stein's method in dimension one. In particular, the results of \cite{Bordino2023} are based on the {\bf improved second-order Poincar\'e inequalities} established in \cite{vidotto2020second}. Similarly, in the case $|J|=0$ and $L\geq 1$ arbitrary, our estimates on ${\bf W}_1$ strictly improve those that can be deduced by combining \cite[Theorem 1.1]{basteri2022quantitative} with the general relation ${\bf W}_1\leq {\bf W}_2$. 
    \item[\rm (d)] In probabilistic approximations, it is typical to measure the distance between the laws of two random variables $X,Y$ by using the so-called {\bf Kolmogorov distance}, which is defined as 
    \begin{equation}\label{e:kolmogorov}
    d_K(X,Y) := \sup_{t\in \mathbb{R}} \left| \mathbb{P}(X>t) -\mathbb{P}(Y>t) \right|.
    \end{equation}
    We observe that $d_{TV}\geq d_K$ so that, in particular, our bound \eqref{eq:one-d-tv-1} implies an estimate on the Kolmogorov distance $d_K(V^J_\alpha z_{i;\alpha}^{(L+1)}, Z)$ that is strictly better than the one implied by the standard relation $d_K(V^J_\alpha z_{i;\alpha}^{(L+1)}, Z) \leq c \sqrt{{\bf W}_1(V^J_\alpha z_{i;\alpha}^{(L+1)}, Z)} $, with $c$ an absolute constant -- see \cite[Remark C.22]{nourdin2012normal}. We refer the reader to \cite[Appendix C]{nourdin2012normal}, and the references therein, for further details on probabilistic distances.\\

\end{itemize}

\begin{obs}{\rm 
\begin{itemize}
\item[(a)] If the infinite-width covariance structure $\{K^{(\ell)}\}$ is {\it degenerate} on the singleton $\set{x_\alpha}$, then $Z = 0$, a.s.-$\mathbb{P}$, and one has the standard estimate
\begin{equation}\label{e:gb}
{\bf W}_1(V^J_\alpha z_{i;\alpha}^{(L+1)}, Z)\leq {\bf Var}(V^J_\alpha z_{i;\alpha}^{(L+1)})^{1/2},
    \end{equation}
    where the right-hand side of \eqref{e:gb} converges to zero at a speed that, in principle, depends on the choice of the parameters $\mathcal{P}$. 
\item[(b)] Point {\bf (2)} of Theorem \ref{thm:one-d} shows that for any choice of the activation $\sigma$ yielding a non-degenerate covariance structure, it is impossible to replace the estimate \eqref{eq:one-d-tv-1} with an upper bound converging to zero as $o(n^{-1})$. 
\end{itemize}
}
\end{obs}

\subsection{Finite-dimensional bounds}\label{s:fd} We now report what happens at the level of finite-dimensional distributions. For this, we recall that, for all integers $m\geq 1$, the {\bf convex distance} between the distributions of two $m-$dimensional random vectors $X$ and $Y$ is
\begin{equation}\label{e:dconv}
d_c (X,Y) := \sup_{B} \left| \mathbb{P}(X\in B) -  \mathbb{P}(Y\in B) \right|,
\end{equation}
where the supremum runs over all convex $B\subset \mathbb{R}^m$. 

\medskip

\begin{obs}{\rm The convex distance $d_c$ is a natural generalization of \eqref{e:kolmogorov} in a multivariate setting. Another popular distance for measuring the discrepancy between the distributions of random vectors is the so-called {\bf multivariate Kolmogorov distance}, which is obtained from \eqref{e:dconv} by restricting the supremum to the class of all hyperrectangles $R\subset \mathbb{R}^m$ (see e.g. \cite{fang2021clt} and the references therein). Our choice of $d_c$ over the multivariate Kolmogorov distance is motivated by the following two features: (i) $d_c$ is invariant with respect to orthogonal and affine transformations, and (ii) $d_c$ can be directly connected to multivariate transport distances through remark (b) below. In particular, property (i) will allow us to compare the distribution of the output of a given neural network and that of a possibly singular Gaussian random vector, see the forthcoming Theorem \ref{thm:finite-d}-{\bf (2)} and its proof. We refer the reader to \cite{bentkus2004clt, gotze1991clt} for some classical examples of the use of $d_c$ in the context of the multivariate CLT, as well as to \cite{nourdin2022multivariate, kasprzak2022convex} for a discussion of recent developments.}

\end{obs}

\medskip

{\begin{thm}\label{thm:finite-d}
Let $x_\alpha\in \R^{n_0}\mapsto z_{\alpha}^{(L+1)}\in \R^{n_{L+1}}$ be a random neural network verifying Assumption \ref{a:cadre} with a non-linearity $\sigma$ that is polynomially bounded to order $r\geq 1$ as in Definition \ref{D:poly-bdd}; recall notation \eqref{e:P}. Fix $m\geq 1$, a set $\mathcal{A}=\{\alpha_1,\ldots, \alpha_m\}$, a finite collection of distinct non-zero network inputs
 \[
 \{ x_{\alpha} : \alpha \in \mathcal{A}\}\subseteq \mathbb{R}^{n_0}
 \]
 and a collection of directional derivatives $V = \{V_1,...,V_p\}$ as in \eqref{e:V}. Further, consider a family ${\bf B} = \{ (J_k, \alpha_k) : k = 1,...,M\}$ of distinct pairs such that $M\geq 2$, where $J_k\in \mathbb N^{p}$ is a multi-index verifying $|J_k|\leq r$ and $\alpha_\ell\in \mathcal{A}$. Finally, for any multi-index $J=\lr{j_1,\ldots, j_p}\in \mathbb N^p$, use the notation \eqref{e:differentials} and set
 \[
 G :=\lr{V_{\alpha_k}^{J_k} \Gamma^{(L+1)}_{i;\alpha_k}}_{\substack{1\leq i \leq n_{L+1}\\  (J_k,\alpha_k) \in {\bf B}}}\in \R^{M\times n_{L+1}},
 \]
 where $\mathbb{R}^{n_0} \ni x_\alpha \mapsto  (\Gamma^{(L+1)}_{1;\alpha},...,\Gamma^{(L+1)}_{n_{L+1};\alpha})$ is the centered Gaussian field with covariance
 \[
\mathrm{Cov}\lr{\Gamma^{(L+1)}_{i;\alpha},\Gamma^{(L+1)}_{j;\beta}}= \delta_{ij}K_{\alpha\beta}^{(L+1)},
 \]
as defined in \eqref{eq:K-rec}.
\begin{enumerate}
    \item[\bf (1)] Suppose the infinite width covariance structure $\{K^{(\ell)}: \ell = 1,...,L+1\}$ is non-degenerate to the order $r$ on $\set{x_\alpha: \alpha \in\mathcal A}$ with respect to $V$, in the sense of Definition \ref{d:nondeg}. 
 Then, the covariance matrix of $G$ is invertible, and there exists a constant $C_0>0$ depending on $\sigma,V, r, {\bf B}, \mathcal{P}$ such that
 \begin{equation}\label{eq:dc-bound}
 d_c\lr{\lr{V_{\alpha_k}^{J_k} z^{(L+1)}_{i;\alpha_k}}_{\substack{1\leq i \leq n_{L+1}\\  (J_k,\alpha_k) \in {\bf B}}}, G} \leq C_0\, n^{-1/2},
 \end{equation}
where we have implicitly regarded $\lr{V_{\alpha_k}^{J_k} z^{(L+1)}_{i;\alpha_k}}_{\substack{1\leq i \leq n_{L+1}\\  (J_k,\alpha_k) \in {\bf B}}}$ and $G$ as $(M\cdot n_{L+1})$-dimensional random vectors.

\item[\bf (2)] Assume that the non-linearity $\sigma$ is smooth ($\sigma\in C^\infty(\mathbb{R})$). Then, there exists a constant $C_1>0$ depending on $\sigma,V, r, {\bf B}, \mathcal{P}$ such that
 \begin{equation}\label{eq:dc-bound-2}
 d_c\lr{\lr{V_{\alpha_k}^{J_k} z^{(L+1)}_{i;\alpha_k}}_{\substack{1\leq i \leq n_{L+1}\\  (J_k,\alpha_k) \in {\bf B}}}, G'} \leq C_1\, n^{-1/2},
 \end{equation}
where
\[
 G':=\lr{V_{\alpha_k}^{J_k} \Gamma'_{i;\alpha_k}}_{\substack{1\leq i \leq n_{L+1}\\  (J_k,\alpha_k) \in {\bf B}}}\in \R^{M\times n_{L+1}},
 \]
and $\mathbb{R}^{n_0} \ni x_\alpha \mapsto  (\Gamma'_{1;\alpha},...,\Gamma'_{n_{L+1};\alpha})$ is the centered Gaussian field with covariance
 \[
\mathrm{Cov}\lr{\Gamma'_{i;\alpha},\Gamma'_{j;\beta}}= \delta_{ij}\mathbb{E}[\Sigma^{(L)}_{\alpha\beta}] = \delta_{ij}\kappa^{(L+1)}_{\alpha\beta},
 \]
with $\Sigma_{\alpha\beta}^{(L)}$ and $\kappa^{(L+1)}_{\alpha\beta}$ defined according to \eqref{eq:Sig-def} and \eqref{e:kappa}, respectively.
\end{enumerate}
\end{thm}

}

\bigskip
We prove Theorem \ref{thm:finite-d} in \S \ref{sec:finite-d}. Before providing in the next section our infinite-dimensional results, we make several remarks, where we write for simplicity $$\lr{V^{J_k} z^{(L+1)}_{i;\alpha_k}} := \lr{V_{\alpha_k}^{J_k} z^{(L+1)}_{i;\alpha_k}}_{\substack{1\leq i \leq n_{L+1}\\  (J_k,\alpha_k) \in {\bf B}}},$$
and also
$$\lr{z^{(L+1)}_{i;\alpha_k}} := \lr{ z^{(L+1)}_{i;\alpha_k}}_{\substack{1\leq i \leq n_{L+1}\\  (0,\alpha_k) \in {\bf B}}},$$
in the case where all multi-indices $J_k$ equal zero (no derivatives).

\begin{itemize}\item[(a)] Under the assumptions of Point {\bf (2)} of Theorem \ref{thm:finite-d}, one might have that the covariance matrix of the vector $G$ is singular. In this case, the law of $G$ is supported by a lower-dimensional subspace $\mathcal{L}\subset \mathbb{R}^{M\cdot n_{L+1}}$, and, in principle, one might have that $\mathbb{P}\left[\lr{V^{J_k} z^{(L+1)}_{i;\alpha_k}}\in \mathcal{L}\right] = 0$ for any choice of $n_1,...,n_L$, which would imply in turn
$$
d_c\lr{\lr{V^{J_k} z^{(L+1)}_{i;\alpha_k}}, G} = 1.
$$
This difficulty is resolved by replacing $G$ with a vector $G'$ having the same covariance matrix as $\lr{V^{J_k} z_{i;\alpha_k}}$. Note that $G$ and $G'$ are two $(M\cdot n_{L+1})$--dimensional centered Gaussian vectors. Using e.g. \cite[Theorem 1.2]{nourdin2022multivariate} in combination with the results stated in Theorem \ref{thm:var-of-var}, one deduces the following rough estimate: under the assumptions of Theorem \ref{thm:finite-d}-{\bf (1)}, one has that $d_c(G,G') = O(n^{-1/2})$, as $n\to\infty$.    \\
    \item[(b)]  Roughly five months after the first version of the present work appeared on ArXiv, D. Trevisan posted reference \cite{trevisan2023wide} on the same server. In that paper it is shown that, for a Lipschitz activation $\sigma$ and under non-degeneracy assumptions analogous to those in Point {\bf (1)} of Theorem \ref{thm:finite-d} (when no derivatives are considered), one can obtain a bound of the order $n^{-1}$ on the quantity ${\bf W}_2 \big( (z^{(L+1)}_{i;\alpha_k}) \, , \,  G\big)$ --- see Definition \ref{d:wass}. Such a remarkable result substantially refines \cite{basteri2022quantitative}. It is interesting to notice that, if the covariance matrix of $G$ is diagonal (e.g. when ${\bf B} = \{(0,\alpha_0)\}$ is a singleton), then one can combine the findings of \cite{trevisan2023wide} with \cite[Proposition 1.4]{zhai2018} and improve the rate of convergence of the righ-hand side of \eqref{eq:dc-bound} from $n^{-1/2}$ to $n^{-2/3}$.  
    As discussed in Section \ref{sec:finite-d}, our findings are based on some refinements of the bounds established in \cite{nourdin2022multivariate} by means of the so-called {\it Stein's method} for multivariate approximations, whereas \cite{basteri2022quantitative, trevisan2023wide} exploit optimal transport arguments, close to those that we will use in an infinite-dimensional setting (see the forthcoming Section \ref{s:functional}). Further quantitative relations between Wasserstein and convex distances can be found in \cite[Proposition A.1]{nourdin2022multivariate}. Finally, we observe that the one-dimensional lower bound \eqref{eq:lowerb} on ${\bf W}_1$ implies that the dependence on $n$ in the upper bounds established in \cite{trevisan2023wide} cannot be improved.  \\

    \item[(c)] ({\it Dimensional dependence}) An inspection of the arguments rehearsed in our proofs (see Section \ref{sec:infinite-d}) reveals the following estimates: (i) the constant $C_0$ in \eqref{eq:dc-bound} is such that 
    \begin{equation}\label{e:constant M}
C_0\leq a_0 (M\cdot n_{L+1})^{65/24},
    \end{equation}
    where $a_0$ depends on $\sigma, L, n_0, C_b, C_W$;  (ii) the constant $C_1$ in \eqref{eq:dc-bound-2} is such that $$
C_1\leq a_1 \lambda_+^{-3/2} R^{65/24},
    $$
    where the constant $a_1$ depends on $\sigma, L, C_b, C_W, n_0$, the symbol $R$ denotes the rank of the covariance matrix of $G'$, and $\lambda_+$ is the smallest strictly positive eigenvalue of the covariance matrix of $G'$. This implies in turn the following rough estimate: if $\lambda_+ $ is bounded away from zero as $M\to \infty$, then $C_1 = O( M^{65/24})$. The exponent $65/24$ does not carry any specific meaning: it is an artifact of the (recursive) techniques used in \cite{nourdin2022multivariate} and can be improved in special situations. To see this, consider for instance the elementary situation where $L=1$ and ${\bf B} = \{(0, \alpha_0)\}$ (so that $|{\bf B}| = 1$). In this case, $\lr{V^{J_\ell} z^{(2)}_{i;\alpha_\ell}} =( z^{(2)}_{i;\alpha_0})_{1\leq i \leq n_{2}}$
    has the law of some multiple of a random vector with the form
    $$
    Z_{n_1} = \frac{1}{\sqrt{n_1}}\sum_{k=1}^{n_1} Y_k,
    $$
    where the $\{Y_k\}$ are i.i.d. centered $n_2$-dimensional {\teal random} vectors with identity covariance matrices. Now, assuming for simplicity that the non-linearity $\sigma $ is bounded, one deduces from \cite{bentkus2004clt, gotze1991clt} (and some elementary computations left to the reader) that, denoting by $Z$ a standard $n_2$-dimensional Gaussian vector,
    $$
    d_c(Z_{n_1}, Z) \leq B\, n_1^{-1/2}, 
    $$
    where $B = O(n_2^{7/4})$ as $n_2\to\infty$ (and the implicit constants are absolute). We also observe that, in the case where ${\bf B}= \{(0, \alpha_0)\}$ (one input, no derivatives), the exponent $65/24$ implicitly appearing in our bound can be reduced to $53/24$. {\teal A further interesting issue (that we leave open for further research) is to compare the dimensional dependence in our bounds with that of the estimates established in \cite{trevisan2023wide}}. \\
    \item[(d)] We provide here one important example in which the infinite width covariance structure  fails to be canonically non-degenerate but is instead non-degenerate with respect to a particular set of directional derivatives. Specifically, consider
    \[
\sigma(t)=\mathrm{ReLU}(t)=\max\set{0,t},\quad C_b=0,\quad C_W=2,\quad r=1,
    \]
    and fix a network input $x_\alpha$ with $\norm{x_\alpha}=1$. Note that, since $x_\alpha \mapsto z_{i;\alpha}^{(\ell)}$ is homogeneous of degree one with respect to $x_\alpha$, we have
\[
z_{i;\alpha}^{(\ell)} = (x_\alpha\cdot \nabla) z_{i;\alpha}^{(\ell)}.    
\]
A direct computation now shows that for any $v_1,v_2\in \R^{n_0}$
\begin{equation}\label{eq:relu-homog}
V_j = v_j \cdot \nabla = \sum_{i=1}^{n_0}(v_j)_i \partial_{x_i}\quad \Rightarrow \quad V_j V_k K_{\alpha\alpha}^{(\ell)} = \frac{2}{n_0}\inprod{v_j}{v_k}.    
\end{equation}
Thus, writing $\partial_{x_0}=\text{id}$ we have
\[
\lr{\partial_{x_i}\partial_{x_j}K_{\alpha\alpha}^{(\ell)}}_{i,j=0,\ldots, n_0} = \frac{2}{n_0}\mathrm{Gram}\lr{x_\alpha, e_1,\ldots, e_{n_0}},
\]
where $e_j$ is the $j-$th standard unit vector. The Gram matrix is not invertible since the vectors are not linearly independent. In contrast, $V=\lr{V_1,\ldots, V_{n_0-1}}$ to be partial derivatives in any set of directions that are a basis for the orthogonal complement to $x_\alpha$, we see from \eqref{eq:relu-homog} that  the infinite width covariance structure is indeed  non-degenerate to order $1$ on $\set{x_\alpha}$ with respect to $V$.\\
\item[(e)]  Similarly to the results obtained in \cite{basteri2022quantitative}, and as noted above in (c), the bounds we obtain diverge when the dimension $M$ increases. It is hence natural to focus on functional bounds, which are indeed addressed in the next subsection.\\

\item[(f)] In the case where ${\bf B} = \{(0, \alpha_0)\}$ (one input, no derivatives), our results are comparable to \cite[Theorem 6.5]{torrisi23}, where some slightly different set of assumptions on the nonlinearity $\sigma$ is considered.\\

\item[(g)] It is natural to conjecture that the rate of convergence $n^{-1/2}$ in our multidimensional estimates \eqref{eq:dc-bound}--\eqref{eq:dc-bound-2} could be improved to $n^{-1}$ as in the one-dimensional case (see Theorem \ref{thm:one-d}). However, in order to establish such a result, one should prove a multidimensional version of Proposition \ref{p:ibp}, a task that seems to be outside the scope of current techniques (the crucial technical difficulty being that no existing version of the multidimensional Stein's method allows one to directly deal with the mappings emerging from the use of Lusin's theorem -- see \cite[Chapters 4 and 6]{nourdin2012normal} for a discussion).\\
 
{\teal \item[(h)] Removing the non-degeneracy assumption from Point {\bf (1)} of Theorem \ref{thm:finite-d} would {\it not} allow one to deduce a bound commensurate to the right-hand side of \eqref{eq:dc-bound}, even considering smooth test functions of class $\mathcal{C}^2$ by using e.g. \cite[Theorem 6.1.2]{nourdin2012normal}. Indeed, while it is true that \cite[Theorem 6.1.2]{nourdin2012normal} makes it possible to consider degenerate Gaussian limits by bypassing the use of Stein's method, one would still need the full power of the forthcoming Theorem \ref{thm:var-of-var} -- for which non-degeneracy assumptions are paramount -- to obtain estimates of the order $n^{-1/2}$.

 }
\end{itemize}

\subsection{Functional Bounds}\label{s:functional} 

For the rest of the section, we let $\R^{n_0}\ni x_\alpha \mapsto z_{\alpha}^{(L+1)}\in \R^{n_{L+1}}$ be a random neural network verifying Assumption \ref{a:cadre}; in particular, $\sigma$ is polynomially bounded to the order $r\geq 1$. To simplify the discussion, from now on we will use the symbol $\mathcal{M}_r$ to denote the class of all multi-indices $J\in \mathbb{N}^{n_0}$ such that $|J|\leq r$; with such a notation, one has that $\mathcal{M}_0 = \{0\}$. For the rest of the section, $\mathbb{U}$ is an open ball contained in $\R^{n_0}$. 

\begin{obs}[Spaces of smooth functions]\label{r:funcspaces}{\rm

\begin{itemize}

\item[(a)] For $k\geq 0$, we write $C^k(\mathbb{U}; \R^{n_{L+1}}) :=C^k(\mathbb{U})$ to indicate the class of $\R^{n_{L+1}}$-valued, $k$ times continuously differentiable functions on $\mathbb{U}$. We also denote by $C^k_b(\mathbb{U})$ the subspace of $C^k(\mathbb{U})$ composed of functions whose derivatives of order $\leq k$ are bounded and uniformly continuous on $\mathbb{U}$. It is a well-known fact (see e.g. \cite[Section 1.3]{demengel2012pdes}) that the elements of $C^k_b(\mathbb{U})$, as well as their derivatives, admit continuous extensions to the closure $\bar{\mathbb{U}}$. It follows that $C^k_b(\mathbb{U})$ can be identified with the space $C^k(\bar{\mathbb{U}})$, that we endow with the supremum norm, defined as follows: for $f = (f_1,...,f_{n_{L+1}}) \in C^k(\bar{\mathbb{U}})$,
\begin{equation}\label{e:supnorm}
\| f\|_{C^k(\bar{\mathbb{U}})}:=\max_{i\in[n_{L+1}]} \max_{J\in \mathcal{M}_r} \max_{x\in \bar{\mathbb{U}}} | D^J f_i(x)|.
\end{equation}
It is clear that the space $C^k(\bar{\mathbb{U}})$ is Polish. In this paper, we will sometimes use the following fact (whose proof can be deduced along the lines of \cite[proof of Lemma A.2]{nazarov2016nodal}): for every $k\geq 0$ and every $m>k$, the set $C^{m}(\bar{\mathbb{U}})$ is a Borel subset of $C^{k}(\bar{\mathbb{U}})$. Analogous conventions and results hold for the spaces $C^{k,k}({\mathbb{U}}\times {\mathbb{U}})$ and  $C^{k,k}(\bar{\mathbb{U}}\times \bar{\mathbb{U}})$, $k\geq 0$.

\item[(b)] For $s\geq 0$, we define $$\mathbb{W}^{s,2}(\mathbb{U}) :=\mathbb{W}^{s,2}(\mathbb{U}; \R^{n_{L+1}}) $$ to be the Sobolev space obtained as the closure of square-integrable $\mathbb{R}^{n_{L+1}}$-valued mappings on $\mathbb{U}$ with square-integrable (weak) derivatives up to the order $s$, see \cite[Chapter 3]{adams2003sobolev}. We observe that $\mathbb{W}^{s,2}(\mathbb{U})$ is a closed subspace of the Hilbert space 
\begin{equation}\label{e:grandh}
H:= L^2\big(\mathcal{M}_s\times [n_{L+1}]\times \mathbb{U},d\nu_0\otimes d\nu_1 \otimes dx\big)
\end{equation}
(from which it inherits the inner product), where $[n] := \{1,...,n\}$, and $\nu_0$ and $\nu_1$ are the counting measures on $\mathcal{M}_s$ and $[n_{L+1}]$, respectively; plainly, for $s=0$, the spaces $\mathbb{W}^{0,2}(\mathbb{U})$ and $H=L^2\big( [n_{L+1}]\times \mathbb{U}, d\nu_1 \otimes dx\big)$ coincide. 
\item[(c)] For every $r\geq 0$, there exists a canonical continuous injection $\iota$ from $C^{r}(\bar{\mathbb{U}})$ to $\mathbb{W}^{r,2}(\mathbb{U})$. This implies that, if $X$ is a random element with values in $C^{r}(\bar{\mathbb{U}})$, then $\iota(X)$ is a well-defined random element with values in $\mathbb{W}^{r,2}(\mathbb{U})$. For the sake of brevity, we will often refer to this fact by saying that ``$X$ is regarded as a random element with values in $\mathbb{W}^{r,2}(\mathbb{U})$'' (or some equivalent formulation), and write $X$ instead of $\iota(X)$ by a slight abuse of notation. Similar conventions are tacitly adopted to deal with the spaces $C^{r,r}(\bar{\mathbb{U}}\times \bar{\mathbb{U}})$ and $\mathbb{W}^{r,2}(\mathbb{U})\otimes \mathbb{W}^{r,2}(\mathbb{U})$.
\item[(d)] We will use the following special consequence of the {\bf Sobolev embedding theorem} (as stated e.g. in \cite[Theorem 2.72]{demengel2012pdes} or \cite[Lemma 4.3]{dierickx2023small}): if $\mathbb{U}$ is an open ball (or, more generally, a Lipschitz domain) and $u\in C^\infty(\bar{\mathbb{U}}):= \bigcap_k C^k(\bar{\mathbb{U}})$, then, for all $k\geq 1$,
\begin{equation}\label{e:sobolev}
\| u \|_{C^k(\bar{\mathbb{U}})} \leq A \cdot \| u \|_{\mathbb{W}^{r,2}({\mathbb{U}})},
\end{equation}
where $r := k+1+\lfloor \frac{n_0}{2} \rfloor$, $\lfloor y \rfloor$ stands for the integer part of $y$, and $A$ is an absolute constant uniquely depending on $\mathbb{U}$.
\end{itemize}

}

\end{obs}

\subsubsection{Random fields as random elements} 

Given the ball $\mathbb{U}\subset \R^{n_0}$, we define the random field
\begin{equation}\label{e:zd0}
z^{(L+1)}_{\mathbb{U}} := \{z^{(L+1)}_{i;x_\alpha} : i\in [n_{L+1}]\, ; \, x_\alpha\in \mathbb{U}\}.
\end{equation}
Our aim is to compare the law of $z^{(L+1)}_{\mathbb{U}}$ with that of 
\begin{equation}\label{e:biggamma}
\Gamma^{(L+1)}_{\mathbb{U}} =\{\Gamma^{(L+1)}_{i;\alpha} : i\in [n_L] \,;\, x_\alpha\in \mathbb{U}\},
\end{equation}
where, as before, $x_\alpha \mapsto (\Gamma_{1;\alpha}^{(L+1)}, ...,\Gamma_{n_{L+1};\alpha}^{(L+1)})$ is the centered Gaussian field with covariance
\begin{equation}\label{e:biggammak}
\mathbb{E}(\Gamma^{(L+1)}_{i;\alpha}\Gamma^{(L+1)}_{j;\beta}) = \delta_{ij} K^{(L+1)}_{\alpha\beta},
\end{equation}
with $K^{(L+1)}$ recursively defined according to \eqref{eq:K-rec}. In view of Remarks \ref{r:peaceofmind}, \ref{r:peaceofmind2} and \ref{r:funcspaces}-(c), $z^{(L+1)}_{\mathbb{U}}$ and $\Gamma^{(L+1)}_{\mathbb{U}}$ can be regarded as both $C^{q}(\mathbb{U})$- and $\mathbb{W}^{q;2}(\mathbb{U})$-valued random elements, for all $q\in \{0,...,r-1\}$. The case $q=r$ is more delicate, however, and we will sometimes make the following assumption:


\begin{asspt}\label{a:sobolev} {\rm {{}The domain $\bar{\mathbb{U}}$ does not contain the origin}, and the non-linearity $\sigma$ is polynomially bounded to the order $r\geq 1$, and both $z^{(L+1)}_{\mathbb{U}}$ and $\Gamma^{(L+1)}_{\mathbb{U}}$ are random elements with values in $\mathbb{W}^{r;2}(\mathbb{U})$ with probability $1$.}
\end{asspt}
\noindent Though we do not know if it always holds, this assumption is easy to verify in the following three cases:
\begin{itemize}
    \item[(i)] $\sigma$ is smooth (i.e. in $C^\infty(\R)$) and $r\geq 1$ is arbitrary.
    \item[(ii)] $\sigma$ is ReLU or leaky ReLU and $r=1$.
    \item[(iii)] $\sigma$ is any non-linearity that is polynomially bounded to the order $r\geq 1$ and $C_b>0$. 
\end{itemize}
Indeed, case (i) is trivial. Case (ii) follows from the observation that the network function is continuous and piecewise linear subordinate to a finite partition of the input space $\R^{n_0}$ into a finite collection of convex polyehdra on which the network is affine. And case (iii) follows from the fact when $C_b>0$ the set of inputs $x_\alpha$ for which $\sigma$ is not $r$-times differentiable at  $z_{i;\alpha}^{(\ell)}$ for some $i,\ell$ has Lebesgue measure $0$. With this in mind, from now on, for an integer $q\leq r-1$, let us write ${\bf W}_{2;q}$ to indicate the distance ${\bf W}_2$ defined in \eqref{e:wassdef} for $K = \mathbb{W}^{q,2}(\mathbb{U})$ (and we extend this definition to $q=r$ when Assumption \ref{a:sobolev} holds). 






\subsubsection{Bounds in Sobolev spaces} 
For {\teal $q\leq r-1$}, we canonically associate with the covariance \eqref{e:biggammak} the (symmetric and positive) trace-class operator ${\bf K} : \mathbb{W}^{q,2}(\mathbb{U})\to \mathbb{W}^{{{}q},2}(\mathbb{U})$ given by
\begin{eqnarray}\label{e:hsop}
&& \mathbb{W}^{q,2}(\mathbb{U})\ni h = \big\{h_i(x) : i\in [n_{L+1}], x\in \mathbb{U} \big\} \mapsto {\bf K}h \\
&& := \left\{ ({\bf K}h)_j(x_\beta) = \sum_{J\in \mathcal{M}_q}\int_{\mathbb{U}} D_\alpha^J h_j(x_\alpha) D^J_\alpha  K^{(L+1)}_{\alpha\beta} dx_{\alpha} : j \in [n_{L+1}], x_\beta \in \mathbb{U}  \right\},\notag
\end{eqnarray}
and denote by 
$$
\lambda_{1;q}^{(L+1)}\geq \lambda_{2;q}^{(L+1)}\geq \cdots  \geq \lambda_{k;q}^{(L+1)} \geq \cdots \geq 0
$$
its eigenvalues. 

{\teal
\begin{obs}\label{r:ref}{\rm For $q\leq r-1$, the covariance kernel appearing on the right-hand side of \eqref{e:biggammak} is an element of $\mathbb{W}^{q,2}(\mathbb{U})\otimes \mathbb{W}^{q,2}(\mathbb{U})$, and can therefore be canonically identified with an element of $H\otimes H$, where $H$ is defined as in \eqref{e:grandh}, for $s=q$. The operator ${\bf K}$ defined in \eqref{e:hsop} coincides therefore with the restriction to $\mathbb{W}^{q,2}(\mathbb{U})$ (regarded as a subspace of $H$) of a classical {\bf kernel operator}, as defined e.g. in \cite[p. 240]{hirschlacombe}. Since the kernel of such an operator is an element of $\mathbb{W}^{q,2}(\mathbb{U})\otimes \mathbb{W}^{q,2}(\mathbb{U})$ it is immediately seen that the image of such an operator is necessarily contained in $\mathbb{W}^{q,2}(\mathbb{U})$. 
    }
\end{obs}
}

\begin{obs}\label{r:hs}{\rm Let $T$ be a generic smooth covariance kernel, and let ${\bf K}_T : \mathbb{W}^{q,2}(\mathbb{U})\to \mathbb{W}^{q,2}(\mathbb{U}) $ be the operator obtained from \eqref{e:hsop} by replacing $K^{(L+1)}$ with $T$. Then, {\teal as explained in Remark \ref{r:ref}, ${\bf K}_T$ can be regarded as the restriction of a kernel operator on $H$.} Exploiting the fact that ${\bf K}_T\, g = 0$ for all $g\in H\cap \mathbb{W}^{{{}q},2}(\mathbb{U})^{\perp} $ {{}(with $H$ defined as in \eqref{e:grandh})}, one deduces that
\begin{equation}\label{e:rhs}
\| {\bf K}_T\|^2_{HS} = \| T\|^2_{\mathbb{W}^{q,2}(\mathbb{U})\otimes \mathbb{W}^{q,2}(\mathbb{U})} :=n_{L+1} \times \sum_{I,J\in \mathcal{M}_q} \int_{\mathbb{U}}\int_{\mathbb{U}} D_\alpha^J D_\beta^I T(x_\alpha,x_\beta)^2 dx_\alpha dx_\beta
\end{equation}
}
\end{obs}
{{}
\begin{obs}\label{r:trace}{\rm For $q\leq r-1$, and adopting the notation \eqref{e:hsop}, one has that
\begin{equation}\label{e:trace}
{\rm Tr}\,({\bf K}) = n_{L+1} \sum_{J\in \mathcal{M}_q}\int_{\mathbb{U}}  \left( D^J_\alpha D^J_\beta  K^{(L+1)}_{\alpha\beta}\, |_{x_\alpha = x_\beta} \right)\,dx_\alpha = \sum_{k\geq 1} \lambda^{(L+1)}_{k;q},
\end{equation}
where ${\rm Tr}\,(\cdot)$ stands for the trace operator. Relation \eqref{e:trace} can be deduced e.g. from \cite[Proposition 1.8 and its proof]{daprato2006}, combined with the fact that -- by virtue of Remark \ref{r:peaceofmind} -- one has that $\mathbb{E}[\|\Gamma_\mathbb{U}^{(L+1)}\|^2_{\mathbb{W}^{q;2}(\mathbb{U})}]< \infty$. 
}\end{obs}
}
In addition to the Wasserstein distances ${\bf W}_{2;q}$ introduced above, we will consider a smoother notion of discrepancy between the distributions of Hilbert space-valued random elements, that we borrow from \cite{bourguin2020approximation}. More precisely, following Bourguin and Campese \cite[Section 3.1]{bourguin2020approximation}, given two random elements $X,Y$ with values in a real separable Hilbert space $K$, we define the distance $d_2$, between the distributions of $X$ and $Y$, to be the quantity { 
\begin{equation}\label{eq:d2-def}
    d_2(X,Y) := \sup_{\substack{g\in C^2_b(K) :\\  \sup_{x\in K}\|\nabla^2 g(x)\|_{K^{\otimes 2}}\leq 1}} |\mathbb{E}[g(X)] - \mathbb{E}[g(Y)]|,
\end{equation}
where $C_b^2(K)$ indicates the class of twice Fr\'echet differentiable real-valued mappings on $K$ with bounded second derivative}; it is a standard fact that $d_2$ metrizes convergence in distribution on $K$. 

\medskip

The following statement contains explicit bounds on the functional Gaussian approximation of random neural networks with arbitrary depth. It is one of the main contributions of our work.

{{}
\begin{thm}\label{thm:infinite-d} Let the above assumptions prevail (in particular, $\sigma$ is polynomially bounded to the order $r\geq 1$), and suppose moreover that the infinite width covariance structure $\{K^{(\ell)}: \ell = 1,...,L+1\}$ is canonically non-degenerate up to the order $q\leq r-1$, in the sense of Definition \ref{d:nondeg}, for all finite subsets $x_\mathcal{A}\subset \mathbb{U}$. Then, one has the following two estimates:
\begin{enumerate}
    \item[\rm\bf (1)] There exists a constant $C>0$ depending on $\sigma,q,\mathbb{U},\mathcal{P}$ (see \eqref{e:P}) such that, 
    \begin{equation}\label{e:fbound1} 
        d_2\left(z^{(L+1)}_{\mathbb{U}} , \Gamma^{(L+1)}_{\mathbb{U}}\right) \leq Cn^{-1/2},
    \end{equation}
where $d_2$ is the distance defined in \eqref{eq:d2-def}, with $K = \mathbb{W}^{q,2}(\mathbb{U})$.

\item[{\bf (2)}] Suppose that 
    \begin{equation}\label{e:eigen}
    \sum_{k=1}^{\infty}\left( \lambda_{k;q}^{(L+1)} \right)^{\frac12}<\infty;
    \end{equation}
    then, there exists a constant $C>0$ depending on $\sigma,q,\mathbb{U},\mathcal{P}$ such that
    \begin{equation}\label{e:fbound1-2} 
        {\bf W}_{2;q}\left(z^{(L+1)}_{\mathbb{U}} , \Gamma^{(L+1)}_{\mathbb{U}}\right) \leq Cn^{-\frac18}.
    \end{equation}

\end{enumerate}  
The conclusions of Points {\bf (1)}--{\bf (2)} hold, more generally, for every $q\leq r$ if one of the following set of assumptions is verified:
\begin{enumerate}
\item[\rm \bf (i)] the non-linearity $\sigma$ is smooth (regardless of any non-degeneracy assumption on the infinite-width covariance structure);
\item[\rm\bf  (ii)] Assumption \ref{a:sobolev} holds, 
the infinite width covariance structure $\{K^{(\ell)}: \ell = 1,...,L+1\}$ is canonically non-degenerate up to the order $r$, and \eqref{e:trace} is verified for $q=r$.

\end{enumerate}

\end{thm}
}
\medskip

\begin{obs}{\rm A careful inspection of the proof reveals that the constants in Theorem \ref{thm:infinite-d} depend on the volume measure of $\mathbb{U}$ and hence, indirectly, on the input dimension $n_0$. To make the comparison with the existing literature more transparent, it is convenient to normalise  $\mathbb{U}$ to have unit measure, analogously to what was done in \cite{Eldan}, \cite{Klukowski} and \cite{CMSV23}, see the discussion below in Example \ref{Example:ReLU} and Section \ref{sec:lit-rev}.}
\end{obs}

\begin{obs}\label{r:zero}{\rm Consider the case in which $r=1$ and the infinite-width covariance structure is canonically non-degenerate to the order $q=0$ on every finite collection of inputs $x_\mathcal{A} \subset \R^{n_0}$. Then, the conclusions of Points {\bf (1)} and {\bf (2)} in Theorem \ref{thm:infinite-d} continue to hold when replacing the ball $\mathbb{U}$ with any bounded subset $\mathbb{V}\subset \R^{n_0}$ (possibly lower dimensional) endowed with some finite measure $\mu$. In this case, one has to interpret $\mathbb{W}^{0,2}(\mathbb{U})$ to be the Hilbert space $L^2(\mathbb{V}, d\mu)$, in such a way that the eigenvalues considered in \eqref{e:eigen} are those of the integral operator on $L^2(\mathbb{V}, \mu) \to L^2(\mathbb{V}, \mu)$ given by
\begin{equation}\label{e:integralV}
h\mapsto \int_{\mathbb{V}} K^{(L+1)}(x,y) h(y) \mu(dy).
\end{equation}

}
\end{obs}

\begin{example}\label{Example:ReLU}
{\rm

If $\sigma(x) = \max\{0,x\}$ (ReLU), then we know that Assumption \ref{a:sobolev} is verified, and that the infinite-width covariance structure is canonically non-degenerate up to the order $q=0$ on every finite collection of inputs $x_\mathcal{A} \subset \R^{n_0}$. We can therefore apply the content of Remark \ref{r:zero} to the case $\mathbb{V} = \mathbb{S}^{n_0-1}$ (sphere) endowed with the unit mass Haar measure. In this case, the  results of \cite[Corollary 2]{bietti2021} {{}imply that \eqref{e:eigen} is true for $q=0$, since $\sum_{k=1}^{\infty}( \lambda_{k;0}^{(L+1)} )^{p}<\infty$ for all $p> 1/(2+n_0)$. This yields that there exists a constant $C>0$ depending on $\sigma,\mathcal{P}$ such that
    \begin{equation*} 
       {{} {\bf W}_{2;0}\left(z^{(L+1)}_{\mathbb{U}} , \Gamma^{(L+1)}_{\mathbb{U}}\right) \leq Cn^{-\frac{1}{8}}.}
    \end{equation*}
This result can be compared with the bounds established in similar circumstances but for shallow networks (i.e., for $L=1$) by \cite{Eldan}, where a logarithmic rate $O((\log n)^{-1})$ is obtained, and by \cite{Klukowski}, whose bound is of order $O(n^{-3/(4n_0-2)})$; see Section \ref{sec:lit-rev} for more discussion and details.} }
\end{example}

\subsubsection{Embedding of smooth non-linearities} In this section, we assume that Assumption \ref{a:cadre} is verified for a certain non-linearity $\sigma \in C^\infty(\mathbb{R})$ that is moreover polynomially bounded to the order $r$, for every $r\geq 1$ (this is equivalent to saying that all derivatives of $\sigma$ are polynomially bounded). In this case, one has that $K^{(L+1)}\in C^{\infty, \infty}( \mathbb{R}^{n_0}\times \mathbb{R}^{n_0})$, and the results stated in \cite[Lemma 4.3]{dierickx2023small} yield that, since $\mathbb{U}$ is an open ball, the estimate \eqref{e:eigen} holds for all $q\geq 0$. Thanks to the last part of Theorem \ref{thm:infinite-d}, this implies in turn that, for all $r\geq 1$, there exists a constant $C>0$ depending on $\sigma,r, \mathcal{P}$ (see \eqref{e:P}) such that
    \begin{equation} \label{e:w2smooth}
       {{} {\bf W}_{2;r}\left(z^{(L+1)}_{\mathbb{U}} , \Gamma^{(L+1)}_{\mathbb{U}}\right) \leq Cn^{-\frac{1}{8}}}.
    \end{equation}

By using \eqref{e:sobolev} one deduces from \eqref{e:w2smooth} some remarkable uniform estimates, that can be expressed in terms of the transport distance ${\bf W}_{\infty; k}$, defined as follows: given two random elements $X,Y$ with values in $C^k(\mathbb{\bar{U}})$, we set
\begin{equation}\label{e:wassinfinity}
 {\bf W}_{\infty; k}(X,Y):= \left( \inf_{(T,S)} \mathbb{E}[\| T-S\|^2_{C^k(\mathbb{\bar{U}})}]\right)^{1/2},  
\end{equation}
where the infimum runs over all random elements $(T,S)$ such that $T\stackrel{law}{=}X$ and $S\stackrel{law}{=}Y$.

\begin{thm}\label{t:sup} Let the assumptions of the present section prevail, and fix $k\geq 1$. Then, there exists a probability space $(\Omega_0, \mathcal{F}_0, \mathbb{P}_0)$ supporting two random elements $X,Y$ with values in $C^{k}(\bar{\mathbb{U}})$ such that
\begin{enumerate}
    \item[\rm (i)] $X\stackrel{law}{=} z^{(L+1)}_{\mathbb{U}}$;
    \item[\rm (ii)] $Y\stackrel{law}{=} \Gamma^{(L+1)}_{\mathbb{U}}$;
    \item[\rm (iii)] There exists a constant $C>0$ depending on $\sigma,r,k, \mathbb{U},\mathcal{P}$ such that
    \begin{equation} \label{e:w2smooth2}    {{}   \mathbb{E}_0\left[\big\|X- Y\big\|^2_{C^k(\mathbb{\bar{U}})} \right]^{1/2} \leq Cn^{-\frac{1}{8} },}
    \end{equation}
    \end{enumerate}
     where $\mathbb{E}_0$ denotes expectation with respect to $\mathbb P_0$. In particular, one has that,
\begin{equation}\label{e:winfi}
{{}{\bf W}_{\infty; k}\left(z^{(L+1)}_{\mathbb{U}},\Gamma^{(L+1)}_{\mathbb{U}}\right)\leq Cn^{-\frac{1}{8}}},
\end{equation}
where the constant $C>0$ is the same as in \eqref{e:w2smooth}.
    
\end{thm}

\smallskip

The previous result implicitly uses the fact that, under the assumptions in the statement, $z^{(L+1)}_{\mathbb{U}}$ and $\Gamma^{(L+1)}_{\mathbb{U}}$ take values in $C_b^k(\mathbb{U})$ (and, consequently, in $C^{k}(\bar{\mathbb{U}})$)  with probability one.

\begin{obs}{\rm

\begin{enumerate}
    \item[(a)]  The results in Theorem \ref{t:sup} can be exploited to obtain useful bounds for the Wasserstein distance between the finite-dimensional distributions of neural networks and their Gaussian limit. Indeed for some $i \in (1,...,n_{L+1})$ consider the two $M$-dimensional vectors $(z_{i;\alpha_k}^{(L+1)})_{k=1,2,...,M}, \, G$; under the previous assumptions and notation, equation \eqref{e:winfi} and some simple algebra yield a bound of the form
\begin{equation}\label{e:multimensional bound}
{\bf W}_2 (z_{i;\alpha}^{(L+1)},G) \leq C \times \sqrt{M} \times \left(\frac1n\right)^{\frac{1}{8}},
\end{equation} 
where ${\bf W}_2$ is the distance defined in \eqref{e:wassdef} for $K=\mathbb{R}^M$. In view of the dependence of the constants in equations \eqref{eq:dc-bound} and \eqref{eq:dc-bound-2} on the dimension of the vectors $(z_{i;\alpha}^{(L+1)},G)$ (see also equation \eqref{e:constant M}), it is immediate to check that, for $M$ large enough with respect to $n$, the bound in \eqref{e:multimensional bound} can be tighter than those in Theorem \ref{thm:finite-d}.

\item[(b)] We believe that the uniform convergence results established in Theorem \ref{t:sup} can open several venues for further research. In particular, these results are the first step to establish weak convergence for geometric and topological functionals of neural networks at the initialization step: for instance, the distribution of their critical points and extrema, the Euler-Poincar\'e characteristic of their excursion sets and their nodal volume. On the one hand, these functionals can provide neat characterizations of the complexity of the neural network landscape; on the other hand, the analytic determination of their expected values and moments is made possible, in the limiting Gaussian case, by classical tools of stochastic geometry such as the Kac-Rice Theorem and the Gaussian Kinematic Formula, see \cite{AzaisWschebor} and \cite{AdlerTaylor}. We leave these topics for further research.


    \item[(c)] In our proofs, we will exploit the following property, valid for all $k\geq 0$ and analogous to \eqref{e:w2cond}: if $X,Y$ are random elements with values in $C^k(\mathbb{\bar{U}})$ and $V$ is a random element defined on the same probability space as $X$ and taking values in some Polish space $\mathcal{V}$ and with law $\mathbb{P}_V$, then
\begin{eqnarray}\label{e:w2inficond}
{\bf W}^2_{\infty; k}(X,Y) \leq \int_\mathcal{V} {\bf W}_{\infty;k}^2 ( \mathbb{P}_{X|V=v}, \mathbb{P}_{Y})\, d\mathbb{P}_V(v),
\end{eqnarray}
see again \cite{basterimasterthesis}.

\end{enumerate}
}
\end{obs}

\section{Related Work}\label{sec:lit-rev}

A few papers have recently addressed quantitative functional central limit theorems for neural networks in the shallow case of $L=1$ hidden layers. More precisely, \cite{Eldan} and \cite{Klukowski} have studied one-hidden-layer neural networks, with random initialization models that are slightly different than ours. In particular, in both cases it is assumed that $x_{\alpha} \in \mathbb{S}^{n_0-1}$; also, their random
weights are assumed to be Rademacher sequences for the second layer (in the special case of polynomial activations, \cite{Klukowski} allows for more general random weights with finite fourth moments). The random coefficients in the inner layer are Gaussian for \cite{Eldan}, uniform on the sphere in \cite{Klukowski}. For activation functions which are polynomials of order $p$, the bounds by \cite{Eldan} and \cite{Klukowski} are respectively of order 
\[
{\bf W}_{2;0}\left (z^{(L+1)}_{\mathbb{S}^{n_0-1}},\Gamma_{\mathbb{S}^{n_0-1}} \right) \leq C\frac{n_0^{5p/6-1/12}}{n^{1/6}} \ , \ {\bf W}_{2;0}\left (z^{(L+1)}_{\mathbb{S}^{n_0-1}},\Gamma_{\mathbb{S}^{n_0-1}} \right) \leq C\frac{(n_0+p)^{n_0}}{n^{1/2}} \ ; 
\]
these rates can be compared with those given above in Theorem \ref{thm:infinite-d}, which for $L=1$ and $\mathbb{U}=\mathbb{S}^{n_0-1}$ yield a decay of square root order, irrespective of the input dimension. In the more relevant case of ReLU nonlinearities, these same authors obtain, respectively:
\[
{\bf W}_{2;0}\left (z^{(L+1)}_{\mathbb{S}^{n_0-1}},\Gamma_{\mathbb{S}^{n_0-1}} \right) \leq C \left ( \frac{\log \log n \times \log n_0}{\log n} \right )^{3/4} \ , \ {\bf W}_{2;0}\left (z^{(L+1)}_{\mathbb{S}^{n_0-1}},\Gamma_{\mathbb{S}^{n_0-1}} \right) \leq 7 \times \frac{1}{n^{3/(4n_0-2)}} \ ; 
\]
comparing to the results discussed above in Example \ref{Example:ReLU} it is easy to see that the bounds in the present paper improve from logarithmic to algebraic compared to \cite{Eldan}, and are exponentially more efficient in the input dimension $n_0$ compared to \cite{Klukowski}. It should also be noted that both \cite{Eldan} and \cite{Klukowski} use cleverly some special properties of the sphere and its eigenfunctions to construct explicit couplings of the neural networks with Gaussian processes; as such, it does not seem seems trivial to generalize their arguments to arbitrary input spaces and/or to the multi-layer framework.

Even more recently, the authors in \cite{CMSV23} have considered one-hidden-layer networks ($L=1$) on the sphere with Gaussian initializations; they have hence exploited functional quantitative central limit results by \cite{bourguin2020approximation} to obtain the following bounds in the $d_2$ norm, in the ReLU and polynomial case, respectively:
\[
d_2\left (z^{(L+1)}_{\mathbb{S}^{n_0-1}},\Gamma_{\mathbb{S}^{n_0-1}} \right) \leq C \left ( \frac{1}{\log n} \right )^{3/4} \ , \ d_2\left (z^{(L+1)}_{\mathbb{S}^{n_0-1}},\Gamma_{\mathbb{S}^{n_0-1}} \right) \leq C \times \frac{1}{n^{1/2}} \ .
\]
The bounds obtained in the present paper are tighter: for instance, for the ReLU case they are algebraic rather than logarithmic in the number of nodes $n$, even when applied to the Wasserstein metric (which, we recall, is strictly stronger than $d_2$). Moreover, the argument in \cite{CMSV23} exploits a direct computation of moments and cumulants which seems difficult to extend to networks of arbitrary depth.

To the best of our knowledge, the only paper so far devoted to quantitative functional central limit theorems for multi-layer networks is \cite{BGRS23}, where the authors establish bounds on the uniform Wasserstein distance between a neural network defined on a sphere, and a Gaussian field with matching covariance. The results in \cite{BGRS23} allow for non-Gaussian weights and hold with respect to rather stringent distance functions. However, the bounds established in \cite{BGRS23} do not apply to the regime $n_1\asymp n_2\asymp \cdots n_{L-1}$ considered in the present paper. As a consequence, a direct comparison with our findings is not possible.

\section{Preparatory results}\label{s:prep}

\subsection{Variance estimates}\label{ss:prepcum}

Our proofs rely extensively on the following estimates from \cite{hanin2022random} on the fluctuations of the  random covariances $\Sigma^{(\ell)}$, defined in \eqref{eq:Sig-def}. In what follows, we write $\kappa(Z_1,...,Z_m)$ to indicate the joint cumulant of random variables $Z_1,...,Z_m$ (see e.g. \cite[Chapter 3]{peccati2011cumulants} for definitions), with the special notation $\kappa_m(Z)$ in the case $Z_1= \cdots = Z_m=Z$.  Note that, according to \eqref{e:kappa} we have $\E[\Sigma]=\kappa$. To avoid overloading the notation of $\kappa$ we prefer to write $\E[\Sigma]$ in the following result.

{\begin{thm}[Thm 3.1, Corollary 3.4, Equation (11.31) in \cite{hanin2022random}]\label{thm:var-of-var}
 Let $x_\alpha\in \R^{n_0}\mapsto z_{\alpha}^{(L+1)}\in \R^{n_{L+1}}$ be a random neural network verifying Assumption \ref{a:cadre} where, for $r\geq 1$, $\sigma$ is polynomially bounded to order $r$ in the sense of Definition \ref{D:poly-bdd}. Fix also a collection of distinct non-zero network inputs
 \[
x_{\mathcal A}:=\set{x_\alpha,\quad \alpha\in \mathcal A}
\]
and directional derivative operators $\{V_1,...,V_p\}$ as in \eqref{e:V}. Suppose that either $\sigma$ is smooth or that the infinite width covariance structure $\set{K^{(\ell)}}$ is non-degenerate to order $q\leq r$ on $x_{\mathcal A}$ with respect to $V$, in the sense of Definition \ref{d:nondeg}. Then, the following asymptotic relations are in order:

\begin{enumerate}
\item[{\bf (1)}] for $\ell=1,\ldots, L$, all multi-indices $J_{1},J_2$ of order at most $q$, and any network inputs $x_{\alpha_1},x_{\alpha_2}\in x_{\mathcal A}$ we have for all $n\geq 1$
\begin{equation}\label{eq:var-n}
\mathrm{max}\set{{\bf Var}({V_{\alpha_1}^{J_1} V_{\alpha_2}^{J_2}\Sigma_{\alpha_1\alpha_2}^{(\ell)}}),\, \abs{V_{\alpha_1}^{J_1}V_{\alpha_2}^{J_2}\left\{\Eof{\Sigma_{\alpha_1\alpha_2}^{(\ell)}}- K_{\alpha_1\alpha_2}^{{}(\ell+1)}\right\}}}|\leq Cn^{-1},
\end{equation}
where for a multi-index $J=\lr{j_1,\ldots, j_p}$ we have used notation \eqref{e:differentials} and have adopted the notational conventions
\begin{eqnarray}\notag
&&V_{\alpha_1}^{J_1} V_{\alpha_1}^{J_2}\Sigma_{\alpha_1\alpha_1}^{(\ell)} := V_{\alpha_1}^{J_1} V_{\alpha_2}^{J_2}\Sigma_{\alpha_1\alpha_2}^{(\ell)}\, |_{x_{\alpha_1 }=x_{\alpha_2 }}, \,\, \mbox{ }  \,\, V_{\alpha_1}^{J_1} V_{\alpha_1}^{J_2}\mathbb{E}[\Sigma_{\alpha_1\alpha_1}^{(\ell)}] := V_{\alpha_1}^{J_1} V_{\alpha_2}^{J_2}\mathbb{E}[\Sigma_{\alpha_1\alpha_2}^{(\ell)}]\, |_{x_{\alpha_1 }=x_{\alpha_2 }},\\
&& V_{\alpha_1}^{J_1}V_{\alpha_1}^{J_2}K_{\alpha_1\alpha_1}^{(\ell)} := V_{\alpha_1}^{J_1}V_{\alpha_2}^{J_2}K_{\alpha_1\alpha_2}^{(\ell)}\, |_{x_{\alpha_1} = x_{\alpha_2}}.\label{e:conv1}
\end{eqnarray} 
The constant $C$ depends on $\alpha_1,\alpha_2,J_1,J_2,\ell,r,q, \mathcal{P}$ (see \eqref{e:P}) but is uniform over $\alpha
_1,\alpha_2$ when the ratios $\norm{x_{\alpha_1}}^2/n_0, \norm{x_{\alpha_2}}^2/n_0$ vary over a compact set.
\item[{\bf (2)}] When $r=1$ and $\mathcal{A} = \{ \alpha\}$ is a singleton, one has also that 
\begin{equation}\label{e:uppercum}
\kappa_3(\Sigma_{\alpha\alpha}^{(\ell)}) \leq C_1 n^{-2}, \quad \mbox{and} \quad \kappa_4(\Sigma_{\alpha\alpha}^{(\ell)}) \leq C_2 n^{-3},
\end{equation}
where the constants $C_1,C_2$ depend on $\alpha,\ell, \mathcal{P}$ and are uniform over $\alpha$ when the ratio $\norm{x_{\alpha_1}}^2/n_0$ varies over a compact set.
\item[{\bf (3)}] Again when $r=1$ and $\mathcal{A} = \{ \alpha\}$ is a singleton, there exist strictly positive constants $B_1, B_2$ and $D_1, D_2$ (depending on $\alpha,\ell, \mathcal{P}$ and uniform over $\alpha$ when the ratio $\norm{x_{\alpha_1}}^2/n_0$ varies over a compact set) such that
\begin{equation}\label{e:exact1}
\left| {\bf Var}({\Sigma_{\alpha\alpha}^{(\ell)}}) - B_1 n^{-1}\right|\leq B_2 \, n^{-2},
\end{equation}
and
\begin{equation}\label{e:exact2}
\Big| \abs{\Eof{\Sigma_{\alpha\alpha}^{(\ell)}}- K_{\alpha\alpha}^{{{}(\ell+1)}} } - D_1 n^{-1} \Big|\leq D_2 \, n^{-2}.
    \end{equation}

\end{enumerate}
\end{thm}
\begin{proof}[Proof Idea]
    Although the proof of Theorem \ref{thm:var-of-var} is somewhat technical, for the sake of completeness we indicate here a few of the key ideas and refer the interested reader to \S 7 of \cite{hanin2022random} for the details. The starting point for the approach is based on the following structural properties of random neural networks:
\begin{itemize}
    \item The sequence of fields $z_{\alpha}^{(\ell)}$ is a Markov Chain with respect to $\ell$. 
    \item Conditional on the sigma algebra $\mF^{(\ell)}$ defined by $z_{\alpha}^{(\ell)}$  the field $z_\alpha^{(\ell+1)}$ is a Gaussian field with independent components $z_{i;\alpha}^{(\ell+1)}$.
    \item The conditional variance $\Sigma_{\alpha\alpha}^{(\ell)}$ of 
    each component $z_{i;\alpha}^{(\ell+1)}$ depends on $z_\alpha^{(\ell)}$ only through random variables of the form
    \[
    \mO_{f}^{(\ell)}:=\frac{1}{n_\ell} \sum_{{{}i}=1}^{n_\ell}f(z_{i;\alpha}^{(\ell)}).
    \]
    The article \cite{hanin2022random} refers to such random variables as collective observables. 
    \item Centered moments of collective observables depend on $n$ as if the random variables $f(z_{i;\alpha}^{(\ell)})$ were independent:
    \begin{equation}\label{E:Delta-concentration}
    \Eof{\lr{\mO_{f}^{(\ell)}-\Eof{\mO_{f}^{(\ell)}}}^q}=O_q\lr{n^{-\lceil \frac{q}{2}\rceil}},\qquad q\geq 0.
    \end{equation}
    Establishing this is the most difficult technical aspect of \cite{hanin2022random}. The basic idea is to proceed by induction on $\ell$. When $\ell=1$, the neuron pre-activations $z_{i;\alpha}^{(1)}$ are independent and hence the estimate \eqref{E:Delta-concentration} is straight-forward. When $\ell \geq 2$, however, the neuron pre-activations {{}$z_{i;\alpha}^{(\ell)}$} are not independent. The idea is to analyze them by first using the law of total cumulance to write cumulants of collective observables in layer $\ell+1$  in terms of cumulants of such objects at layer $\ell$. 
\end{itemize}
Once the estimate \eqref{E:Delta-concentration} is established, it is now fairly straight-forward to study the mean and variance of $\Sigma_{\alpha\alpha}^{(\ell)}$. So let us now explain, mostly dispensing with rigor, how these four ideas come together to obtain a recursive description of the distribution of the field $z_\alpha^{(\ell+1)}$ in terms of that of $z_\alpha^{(\ell)}$ (we stick here to the case of a single input $x_\alpha$). Denoting by $\xi=\lr{\xi_1,\ldots,\xi_m}$ dual variables, consider the characteristic function
\[
p^{(\ell+1)}(\xi):=\Eof{\exp\left[-i\sum_{i=1}^m \xi_i z_{i;\alpha}^{(\ell+1)}\right]}
\]
of $m$ neuron pre-activations $\lr{z_{i;\alpha}^{(\ell)},\, i=1,\ldots, m}$. Conditioning on $z_\alpha^{(\ell)}$ and using conditional Gaussianity allows us to write
\[
p^{(\ell+1)}(\xi):=\Eof{\exp\left[-\frac{1}{2}\norm{\xi}^2 \Sigma_{\alpha\alpha}^{(\ell)}\right]}.
\]
where we note that $\Sigma_{\alpha\alpha}^{(\ell)}$ 
is a collective observable in layer $\ell$. Writing as before
\[
\kappa_{\alpha\alpha}^{(\ell)}:=\Eof{\Sigma_{\alpha\alpha}^{(\ell)}},\qquad \Delta_{\alpha\alpha}^{(\ell)}:=\Sigma_{\alpha\alpha}^{(\ell)} - \Eof{\Sigma_{\alpha\alpha}^{(\ell)}},
\]
we find
\[
p^{(\ell+1)}(\xi):=\Eof{\exp\left[-\frac{1}{2}\norm{\xi}^2 \Delta_{\alpha\alpha}^{(\ell)}\right]}\exp\left[-\frac{1}{2}\norm{\xi}^2\kappa_{\alpha\alpha}^{(\ell)}\right].
\]
The second term is precisely the characteristic function of a centered $m$-dimensional Gaussian with iid components of variance $\kappa_{\alpha\alpha}^{(\ell)}$. Moreover, at least heuristically, the first term can be written
\[
\Eof{\exp\left[-\frac{1}{2}\norm{\xi}^2 \Delta_{\alpha\alpha}^{(\ell)}\right]} = \sum_{q\geq 0} \Eof{\lr{\Delta_{\alpha\alpha}^{(\ell)}}^q} \frac{(-1)^q}{2^qq!} \norm{\xi}^{2q}.
\]
Since 
\[
-\norm{\xi}^2 = \text{ Laplacian in the variables } z_{i;\alpha}^{(\ell+1)},
\]
we have for any reasonable test function $f$ that
\[
\Eof{f(z_{i;\alpha}^{(\ell+1)},\, i=1,\ldots, m)} = \sum_{q=0}^\infty \frac{1}{2^qq!} \Eof{\lr{\Delta_{\alpha\alpha}^{(\ell)}}^q} \bk{\lr{\sum_{i=1}^m \partial_{z_{i;\alpha}}^2}^q f(z_{i;\alpha},\, i=1,\ldots, m)}_{\kappa_{\alpha\alpha}^{(\ell)}},
\]
where $(z_{i;\alpha},\, i=1,\ldots, m)$ is a vector of iid centered Gaussians with variance $\kappa_{\alpha\alpha}^{(\ell)}$. The concentration estimates \eqref{E:Delta-concentration} ensure that this expression is a power series in $1/n$. In particular,
\begin{align}
\label{E:pert-exp-outline}\Eof{f(z_{i;\alpha}^{(\ell+1)},\, i=1,\ldots, m)} &= \bk{ f(z_{i;\alpha},\, i=1,\ldots, m)}_{\kappa_{\alpha\alpha}^{(\ell)}} \\ \notag &+\frac{\Eof{(\Delta_{\alpha\alpha}^{(\ell)})^2}}{8} \bk{\lr{\sum_{i=1}^m \partial_{z_{i;\alpha}}^2}^2 f(z_{i;\alpha},\, i=1,\ldots, m)}_{\kappa_{\alpha\alpha}^{(\ell)}} + O(n^{-2}). 
\end{align}
To derive usable recursions for cumulants of $z_{i;\alpha}^{(\ell+1)}$ in terms of those of $z_{i;\alpha}^{(\ell)}$ one now notes that $\Eof{(\Delta_{\alpha\alpha}^{(\ell)})^2}$ is precisely three times the $4$-th cumulant of $z_{i;\alpha}^{(\ell)}$ (see Lemma \ref{l:cum2}) and takes $f$ to be various polynomials.
\end{proof}

}

The following statement -- used repeatedly in our proofs -- is a direct consequence of Theorem \ref{thm:var-of-var}-{\bf (1)}. As before, we write $\mathcal{M}_r$ to denote the class of all multi-indices $J$ such that $|J|\leq r$.
{{}
\begin{prop}\label{p:integral} Fix a compact domain $\mathbb{T}\subset \mathbb{R}^{n_0}$, and consider a Borel-measurable $\mathbb{U}\subset \mathbb{T}$. Suppose that the assumptions of Theorem \ref{thm:var-of-var} are satisfied for every finite collection of distinct network inputs $x_\mathcal{A} \subset \mathbb{U}$, and let $\mu$ be a finite measure on $\mathcal{M}_q\times \mathbb{U}$. Then, defining
\begin{equation}\label{e:an}
A_n:= \int_{\mathcal{M}_q\times \mathbb{U} }\int_{\mathcal{M}_q\times \mathbb{U} } {\bf Var}\left({V_{\alpha_1}^{J_1} V_{\alpha_2}^{J_2}\Sigma_{\alpha_1\alpha_2}^{(L)}}\right) \mu(dJ_1, dx_{\alpha_1})\mu(dJ_2, dx_{\alpha_2}),
\end{equation}
\begin{equation}\label{e:bn}
B_n:= \int_{\mathcal{M}_q\times \mathbb{U} }\int_{\mathcal{M}_q\times \mathbb{U} } \mathbb{E}\left[\left( V_{\alpha_1}^{J_1} V_{\alpha_2}^{J_2}\Sigma_{\alpha_1\alpha_2}^{(L)} - V_{\alpha_1}^{J_1}V_{\alpha_2}^{J_2} K_{\alpha_1\alpha_2}^{(L+1)}\right)^2\right] \mu(dJ_1, dx_{\alpha_1})\mu(dJ_2, dx_{\alpha_2}),
\end{equation}
and 
\begin{equation}\label{e:cn}
C_n:= \int_{\mathcal{M}_q\times \mathbb{U} } \mathbb{E}\left[\left( V_{\alpha}^{J} V_{\alpha}^{J}\Sigma_{\alpha\alpha}^{(L)} - V_{\alpha}^{J}V_{\alpha}^{J} K_{\alpha\alpha}^{(L+1)}\right)^2\right] \mu(dJ, dx_{\alpha}),
\end{equation}
one has that 
\begin{equation}\label{e:effective}
\max\{A_n, B_n \} \leq D\cdot \mu(\mathcal{M}_q\times \mathbb{U})^2\cdot n^{-1} \quad \mbox{and} \quad C_n \leq D\cdot \mu(\mathcal{M}_q\times \mathbb{U})\cdot n^{-1},
\end{equation}
where the constant $D$ depends on $\mathbb{T}, \mathcal{P},q,r$.
\end{prop}
}

\subsubsection{Connection with output cumulants} The variances appearing in \eqref{eq:var-n} admit a direct interpretation in terms of the cumulants of the network outputs $\{z^{(L+1)}_{i;\alpha}\}$ and their derivatives. To see this, we record the following elementary statement (the proof is left to the reader).

\begin{lemma}\label{l:cum2} Consider a random vector $(X,Y)$ as well as a positive definite $2\times 2$ symmetric random matrix $\Sigma = \{\Sigma (i,j) : 1\leq i,j\leq 2\}$ with square-integrable entries. Assume that, conditionally on $\Sigma$, $(X,Y)$ is a centered Gaussian vector with covariance $\Sigma$. Then, $X$ and $Y$ have finite moments of order 4 and
$$
2{\bf Var} (\Sigma(1,2)) = \kappa(X,X,Y,Y) - {\bf Cov}(\Sigma(1,1), \Sigma(2,2)).
$$
\end{lemma}

Applying Lemma \ref{l:cum2} to $X = V^{J_1}_{\alpha_1} z^{(\ell+1)}_{i;\alpha_1}$ and $Y = V^{J_2}_{\alpha_2} z^{(\ell+1)}_{i;\alpha_2}$, and exploiting Lemma \ref{lem:cond-gauss}, yields the remarkable identity
\begin{eqnarray*}\label{e:cum3}
 2{\bf Var}\left({V_{\alpha_1}^{J_1} V_{\alpha_2}^{J_2}\Sigma_{\alpha_1\alpha_2}^{(\ell)}}\right) &+& {\bf Cov}\left({V_{\alpha_1}^{J_1} V_{\alpha_1}^{J_1}\Sigma_{\alpha_1\alpha_1}^{(\ell)}}, {V_{\alpha_2}^{J_2} V_{\alpha_2}^{J_2}\Sigma_{\alpha_2\alpha_2}^{(\ell)}} \right) \\
&&\quad =\kappa\left(V_{\alpha_1}^{J_1}z^{(\ell+1)}_{i;\alpha_1}, V_{\alpha_1}^{J_1}z^{(\ell+1)}_{i;\alpha_1}, V_{\alpha_2}^{J_2}z^{(\ell+1)}_{i;\alpha_2}, V_{\alpha_2}^{J_2}z^{(\ell+1)}_{i;\alpha_2}\right), 
\end{eqnarray*}
where we have used \eqref{e:conv1}; in particular,
\begin{eqnarray*}\label{e:cum4}
&& 3{\bf Var}\left({V_{\alpha_1}^{J_1} V_{\alpha_1}^{J_1}\Sigma_{\alpha_1\alpha_1}^{(\ell)}}\right) = \kappa_4\left(V_{\alpha_1}^{J_1} z^{(\ell+1)}_{i;\alpha_1}\right).
\end{eqnarray*}

In the next two sections, we will focus on probabilistic bounds based on the so-called {\bf Stein's method} for normal approximations. The reader is referred e.g. to \cite{nourdin2012normal} for a general introduction to this topic.

\subsection{Stein's bounds in dimension 1}\label{ss:preps1} 

Our main tool for one-dimensional probabilistic approximations is the following new estimate on the normal approximation of condtionally Gaussian random variables.

\begin{prop}\label{p:ibp} Let $F$ be a centered random variable with finite variance $\sigma^2>0$, and consider $Z\sim N(0,\sigma^2)$. Assume that there exists an auxiliary integrable random variable $A\geq 0$ such that, conditionally on $A$, the random variable $F$ has a centered Gaussian distribution with variance $A$. Then, for all functions $f : \mathbb{R}\to \mathbb{R}$ continuously differentiable and Lipschitz and every $\varphi : \mathbb{R}_+ \to \mathbb{R}$ bounded,
\begin{eqnarray}\label{e:ibp}
\mathbb{E}[ F f(F) \varphi(A)] = \mathbb{E}[A f'(F) \varphi(A)],
\end{eqnarray}
so that, in particular, $\sigma^2 = \mathbb{E}(A)$. Moreover, the following two properties hold:
\begin{enumerate}

\item[\bf (1)] if $A$ is square-integrable, then 
\begin{eqnarray}\label{e:ktvbound}
d_{TV}(F,Z) &\leq &\frac{8}{\sigma^4} {\bf Var} (A)\\
{\bf W}_1(F,Z)&\leq &  \frac{4}{\sigma^3} {\bf Var} (A);\label{e:wbound}
\end{eqnarray}

\item[\bf (2)] if $\mathbb{E}(A^4) <\infty$, then
\begin{eqnarray}\label{e:lowerbound}
\min\{ 2 d_{TV}(F,Z)\, ;\, {\bf W}_1(F,Z)\} \geq e^{-\sigma^2/2} \left| \frac18{\bf Var}(A) -\frac{1}{48} \mathbb{E}[(A-\sigma^2)^3] +R \right |,
\end{eqnarray}
where $|R| \leq 384^{-1} e^{\sigma^2/2} \mathbb{E}[(A-\sigma^2)^4]$.
\end{enumerate}

\smallskip

\begin{obs}{\rm 

\begin{enumerate}

\item[(a)] By virtue of Lemma \ref{l:cum2}, one has that ${\bf Var}(A) = \frac13 \kappa_4(F)$.

\item[(b)] {\teal If ${\bf Var}(A) = 0$ (that is, if $A=\sigma^2$, a.s.-$\mathbb{P}$), then both sides of \eqref{e:lowerbound} are equal to zero (as it should be). }
\item[(c)]  The random variable $F$ in this Proposition is a Gaussian variance mixture, and we will use this result in connection with Lemma \ref{lem:cond-gauss}.

\end{enumerate}
}
\end{obs}

\smallskip

\end{prop}
\begin{proof}[Proof of Proposition \ref{p:ibp}] Formula \eqref{e:ibp} follows by conditioning and Gaussian integration by parts, and we can consequently focus on the proof of Point {\bf (1)}. Using the fact that the random variable $\widetilde{F} := F/\sigma$ verifies the assumptions in the statement with $\widetilde{A} := A/\sigma^2$, one sees that it is sufficient to only consider the case $\sigma=1$. Combining Stein's method with Lusin's theorem (see \cite[p. 56]{rudin1987analysis}) as in \cite[Lemma 3.1, Proposition 4.16 and Theorem 5.2]{nourdin2013lectures} yields that 
$$
d_{TV}(F,Z) \leq \sup_{f : |f|\leq 1, \, |f'|\leq 2} \left| \mathbb{E}[Ff(F) - f'(F)]\right|,
$$
where the supremum runs over all mappings $f: \mathbb{R}\to\mathbb{R}$ of class $C^1(\R)$ such that $|f|$ and $|f'|$ are bounded by 1 and 2, respectively. Similarly, \cite[Theorem 3.5.2]{nourdin2012normal} yields that
$$
{\bf W}_1(F,Z) \leq \sup_{f :  |f'|\leq 1} \left| \mathbb{E}[Ff(F) - f'(F)]\right|,
$$
where the supremum runs over all mappings $f: \mathbb{R}\to\mathbb{R}$ of class $C^1(\R)$ such that $|f'|$ is bounded by 1. Combining \eqref{e:ibp} with the two estimates above and taking conditional expectations yields that 
$$
d_{TV}(F,Z)\leq 2 \mathbb{E}[|\mathbb{E}(1-A\,|\, F) |], \quad \mbox{and} \quad {\bf W}_1(F,Z)\leq \mathbb{E}\mathbb[|\mathbb{E}(1-A\,|\, F)|]
$$
(recall that, in this part of the proof, $\sigma^2 = 1$ by assumption). The key step (corresponding to a strategy already exploited in \cite[Section 3]{nourdin2015optimal}) is now to observe that
$$
\mathbb{E}[|\mathbb{E}(1-A\,|\, F) |] =\mathbb{E}[\,  {\bf sign}(\mathbb{E}(1-A\,|\, F)) \, \mathbb{E}(1-A\,|\, F)],
$$
so that, by using once again Lusin's theorem in the form of \cite[p. 56]{rudin1987analysis} one deduces that
$$
\mathbb{E}[|\mathbb{E}(1-A\,|\, F) |]\leq \sup_{g\in \mathcal{C} } \left| \mathbb{E}[g(F)(1-A)] \right|,
$$
where the supremum runs over the class $\mathcal{C}$ of all continuous functions $g : \mathbb{R}\to \mathbb{R}$ that have compact support and are such that $|g|\leq 1$. Fix $g\in \mathcal{C}$. Since $\mathbb{E}(A) = 1$, one has that 
$$
\mathbb{E}[g(F)(1-A)] = \mathbb{E}[(g(F)-\mathbb{E}[g(Z)] )(1-A)].
$$
To estimate the right-hand side of the previous equation, we use the classical fact that, according to e.g. to \cite[Proposition 2.1]{nourdin2015optimal}, the differential equation
$$
g(x) - \mathbb{E}[g(Z)] = f'(x) - xf(x),
$$
admits a unique bounded solution $f_g \in C^1(\mathbb{R} )$ such that $|f'_g| \leq 4$. As a consequence, one has that 
\begin{eqnarray*}
&& \mathbb{E}[g(F)(1-A)] = \mathbb{E}[f'_g(F) (1-A) ] - \mathbb{E}[F f_g(F) (1-A) ] \\
&& = \mathbb{E}[f'_g(F) (1-A) ] - \mathbb{E}[f'_g(F) A (1-A)] = \mathbb{E}[f'_g(F) (1-A)^2 ],
\end{eqnarray*}
where in the second equality we have used the fact that $\mathbb{E}[ Ff_g(F) \, |\, A] = A\mathbb{E}[ f'_g(F) \, |\, A]$, by \eqref{e:ibp}. This implies that $|\mathbb{E}[g(F)(1-A)]| \leq 4 {\bf Var}(A)$, and the proof of Point {\bf (1)} is complete. To deal with Point {\bf (2)}, we consider a generic $\sigma^2>0$ and observe that, according e.g. to \cite[Proposition C.3.5]{nourdin2012normal}, 
$$
2 d_{TV}(F,Z) = \sup_{h : |h|\leq 1 } | \mathbb{E}[h(F)] -\mathbb{E}[h(Z)] |, 
$$
where the supremum runs over all Borel measurable functions $h$ whose absolute value is bounded by 1. By virtue of \eqref{e:w1dual} one has therefore that both $2 d_{TV}(F,Z)$ and ${\bf W}_1(F,Z)$ are bounded from below by the quantity
$$
| \mathbb{E}(\cos(F)) - \mathbb{E}(\cos(Z)) | = \left| \mathbb{E}[ e^{-A/2} - e^{-\sigma^2/2} ] \right|
$$
Relation \eqref{e:lowerbound} now follows by writing the Taylor expansion
\begin{eqnarray*}
&& e^{-A/2}- e^{-\sigma^2/2} \\
&& = - e^{-\sigma^2/2} (A/2 - \sigma^2/2) +\frac{e^{-\sigma^2/2}}{2} (A/2 - \sigma^2/2)^2 - \frac{e^{-\sigma^2/2}}{6}(A/2 - \sigma^2/2)^3+ R_0,
\end{eqnarray*}
with $|R_0| \leq \frac{1}{24}(A/2 - \sigma^2/2)^4$, and taking expectations on both sides.
\end{proof}

\medskip

\begin{obs}\label{r:dgauss}{\rm If $Z_1 \sim N(0,\sigma^2_1)$ and $Z_2 \sim N(0,\sigma^2_2)$, then \cite[Proposition 3.6.1]{nourdin2015optimal} implies that
\begin{equation}\label{e:gauss1}
d_{TV}(Z_1,Z_2) \leq \frac{2}{\sigma_1^2\vee \sigma_2^2}\times |\sigma_1^2 - \sigma_2^2|.
\end{equation}
Also, choosing as a coupling $T = \sigma_1 \cdot Z$ and $S = \sigma_2 \cdot Z$, with $Z\sim N(0,1)$, one infers that
\begin{equation}\label{e:gauss2}
{\bf W}_1(Z_1,Z_2) \leq | \sigma_1 - \sigma_2|.
\end{equation}
}
\end{obs}

\subsection{Multidimensional Stein's bounds}\label{ss:preps2}

When dealing with multidimensional normal approximations in the convex distance $d_c$, one has to deal separately with the case of singular and non-singular target covariance matrices.

The next statement deals with the non-singular case; the proof can be deduced by reproducing the arguments leading to the proof of \cite[Theorem 1.2]{nourdin2022multivariate}, and is omitted for the sake of brevity.

\begin{prop}[Convex distance to non-singular Gaussian vectors] \label{p:ibp2} Let $F = (F_1,...,F_M)$ be a centered random vector with square-integrable entries. Assume that there exists an $M\times M$ random matrix $\Sigma = \{\Sigma(i,j) : i,j=1,...,M\}$ with square-integrable entries and such that, for all twice differentiable functions $h : \mathbb{R}^M \to \mathbb{R}$ that are $1$-Lipschitz and such that
$$
\sup_{x\in \mathbb{R}^M}\| {\rm Hess}\, h(x)\|_{HS}  \leq 1,
$$
one has the identity
\begin{equation}\label{e:sibp}
\mathbb{E}[ \langle F , \nabla h(F) \rangle] = \mathbb{E}[ \langle \Sigma, {\rm Hess} \, h(F)\rangle_{HS}]. 
\end{equation}
Then, ${\bf Cov}(F_i, F_j) = \mathbb{E}[\Sigma(i,j)]$, $i,j=1,...,M$. Moreover, denoting by $N = (N_1,...,N_M)$ a centered Gaussian vector with covariance $B>0$, the following estimate is in order:
$$
d_c(F,N) \leq 402\, \{ \lambda_{min}(B)^{-3/2} +1 \} \, M^{41/24}  \sqrt{ \| \Sigma - B\|^2_{HS}},
$$
where $\lambda_{min}(B)$ is the smallest eigenvalue of $B$. 
 \end{prop}

\begin{obs}{\rm In the parlance of \cite{ledoux2015logsob, courtade2019kernel, fathi2019kernel}, any random matrix $\Sigma$ verifying relation \eqref{e:sibp} is a {\bf Stein's kernel} associated with $F$. It is a well-known fact that, for $m\geq 2$, Stein's kernels are in general not unique (see e.g. the discussion contained in \cite{courtade2019kernel}).

}
\end{obs}

The second result of the section is new and allows one to deal with singular covariance matrices in some specific situations (that are relevant to the present paper). The proof uses ideas already exploited in \cite{basteri2022quantitative}.

\begin{prop}[Convex distance to singular Gaussian vectors]\label{p:ibp3} Let $F = (F_1,...,F_M)$ be a centered random vector with square-integrable entries. Assume that there exists a $m\times m$ positive definite symmetric random matrix $\Sigma = \{\Sigma(i,j) : i,j=1,...,M\}$ with square-integrable entries and such that, conditionally on $\Sigma$, $F$ has a centered Gaussian distribution with covariance $\Sigma$. Then, ${\bf Cov}(F_i, F_j) := C(i,j) = \mathbb{E}[\Sigma(i,j)]$, $i,j=1,...,M$. Moreover, denoting by $N = (N_1,...,N_M)$ a centered Gaussian vector with covariance $C$, the following estimate is in order
$$
d_c(F,N) \leq 402\, \{ \lambda_{+}(C)^{-3/2} +1 \} \, \mathrm{rk}(C)^{41/24}  \sqrt{\sum_{i,j=1}^M {\bf Var}(\Sigma(i,j)}),
$$
where $\lambda_{+}(C)$ is the smallest {\rm positive} eigenvalue of $C$ and we have written $\mathrm{rk}(C)$ for the rank of $C$. 
\begin{proof} If $C$ has full rank, then the result follows from Proposition \ref{p:ibp2}. We can therefore assume that $\mathrm{rk}(C) = k<M$. Without loss of generality, we may also assume that $C = U^T D U$, where $U$ is an orthogonal matrix, and $D$ is a diagonal matrix whose diagonal entries $d_i$ are such that $d_i>0$ if $i\leq k$ and $d_i=0$ otherwise. Following a strategy put forward in \cite{basteri2022quantitative}, we now introduce an auxiliary random vector $Z = (Z_1,...,Z_M)$ defined as $Z := U F$. A direct computation shows the following facts:
\begin{itemize}
\item[(i)] conditionally on $\Sigma$, the vector $Z$ is centered and Gaussian with covariance $\Sigma_0 := U\Sigma U^T$;

\item[(ii)] as a consequence, $Z$ is centered with covariance given by the diagonal matrix $D = U C U^T$, which yields $Z_i = 0$, a.s.-$\mathbb{P}$, for all $i>k$, and $\Sigma_0(i,\ell) = 0$, a.s.-$\mathbb{P}$, whenever $\max(i,\ell)>k$;

\item[(iii)] the vector $UN$ is centered and Gaussian, with covariance matrix given by $D$.

\end{itemize}
To conclude the proof, we observe that
$$
d_c(F , N) =d_c(Z, U N) \leq   d_c(Z(k) , U N (k)),
$$
where $Z(k)$ and $U N_0 (k)$ denote, respectively, the first $k$ entries of $Z$ and $UN_0$. Applying Proposition \ref{p:ibp2} to the right-hand side of the previous inequality yields the desired conclusion, by virtue of the relation 
$$
\| \Sigma -  C \|_{HS} = \| U^T \Sigma U - U^T C U\|_{HS} = \| \Sigma_0 - D \|_{HS}, 
$$
where the first equality follows from the unitary invariance of the Hilbert-Schmidt norm.
    
\end{proof}

\begin{obs}\label{r:caffarelli}
{\rm
Let $N_1$ and $N_2$ be $M$-dimensional centered Gaussian vectors with covariances $C_1$ and $C_2$, respectively. Then, choosing the pairing $T = \sqrt{C_1}\, N$ and $S = \sqrt{C_2}\, N$, where $N$ is a standard Gaussian vector, one has the following estimate:
\begin{eqnarray}
\label{e:fdg3} {\bf W}_2(N_1, N_2) \leq  \| \sqrt{C_1} - \sqrt{C_2} \|_{HS}.
\end{eqnarray}
See e.g. \cite{gelbrich1990formula} for optimal bounds. }
\end{obs}


\end{prop}

\subsection{Comparison of Hilbert-space valued Gaussian random elements}\label{ss:prepfunc}

Let $H$ be a separable real Hilbert space, and endow $H$ with the Borel $\sigma$-field associated with the norm $\|\bullet\|_H$. 

We consider two centered, Gaussian $H$-valued random elements $X_1,X_2$, and denote by $S_1$ and $S_2$ their covariance operators. We recall that $S_i$ is the unique symmetric, positive and trace-class linear operator $S_i : H\to H$ such that, for all $g\in H$, $\langle X_i, g\rangle_H$ is a centered Gaussian random variable with variance $\langle S_i g, g\rangle_H\geq 0$ (see e.g. \cite[Chapter 1]{daprato2006}). The following classical bound allows one to compare the distributions of $X_1$ and $X_2$ in the sense of the $2$-Wasserstein distance. It is a direct consequence of Gelbrich \cite[Theorem 3.5]{gelbrich1990formula}; see also \cite{masarotto2019procruste} for a modern discussion of Gelbrich's results.

\begin{prop}[See \cite{gelbrich1990formula, masarotto2019procruste}]\label{p:gelbrich} Let the above assumptions and notations prevail. Then, 
$$
{\bf W}_2(X_1,X_2)\leq \|\sqrt{S_1} - \sqrt{S_2}\|_{HS}.
$$
\end{prop}
{{}
In order to deal with the norm $\|\sqrt{S_1} - \sqrt{S_2}\|_{HS}$ (which is typically not directly amenable to analysis), we will use a variation of the classical {\bf Powers-St\o rmer's inequality} from \cite[Lemma 4.2]{powers1970bounds}, in a form that represents a slight generalization of \cite[Lemma 4.4]{dierickx2023small}. A detailed proof is provided for the sake of completeness.

\begin{prop}\label{p:bello} Under the assumptions of the present section, 
one has that
$$
\|\sqrt{S_1}-\sqrt{S_2}\|_{HS} \leq \left| {\rm Tr}\, (S_1) - {\rm Tr}\, (S_2)\right|^{1/2} +\sqrt{2}  \|S_1-S_2\|_{HS}^{\frac{1}{4}} \, {\rm Min}\{{\rm Tr}\,  (\sqrt{S_1}),{\rm Tr}\,  (\sqrt{S_2})\}^{1/2} .
$$
\end{prop}

\begin{proof}[Proof of Proposition \ref{p:bello}] 
By symmetry, it is enough to prove 
$$
\|\sqrt{S_1}-\sqrt{S_2}\|_{HS} \leq \left| {\rm Tr}\, (S_1) - {\rm Tr}\, (S_2)\right|^{1/2} +\sqrt{2}  \|S_1-S_2\|_{HS}^{\frac{1}{4}} \, {\rm Tr}\,  (\sqrt{S_1})^{1/2} .
$$
For this, let us denote by 
$$
\lambda_1\geq \lambda_2\geq \cdots \geq \lambda_k \geq \cdots\geq 0
$$
the eigenvalues of the operator $S_1$.
We can assume that ${\rm Tr}\,  (\sqrt{S_1}) = \sum_{k=1}^\infty (\lambda_k)^{1/2}<\infty$, because otherwise the inequality is trivial. For all $h\in H$, the action of the operators $S_1$ and $\sqrt{S_1}$ on $h$ can be written, respectively, as $S_1 h = \sum_i \lambda_i \langle e_i, h\rangle_H \, e_i$ and $\sqrt{S_1} h = \sum_i \sqrt{\lambda_i} \langle e_i, h\rangle_H\, e_i$, for some orthonormal basis $\{e_i : i\geq 1\}$ of $H$ such that $e_i$ is an eigenfunction of $S_1$ with eigenvalue $\lambda_i$ (such a basis $\{e_i\}$ is fixed for the rest of the proof).  We start by writing the elementary relation
\begin{eqnarray*}
&& \|\sqrt{S_1} - \sqrt{S_2}\|^2_{HS} \leq \left| {\rm Tr}\, (S_1) - {\rm Tr}\, (S_2) \right| +2 \left| \langle \sqrt{S_1} - \sqrt{S_2}, \sqrt{S_1} \rangle_{HS}\right|.
\end{eqnarray*}
Writing $T:=\sqrt{S_1} - \sqrt{S_2}$, from the definition of the Hilbert-Schmidt norm one infers that
\begin{eqnarray*}
\left| \langle \sqrt{S_1} - \sqrt{S_2}, \sqrt{S_1} \rangle_{HS}\right|&\leq& \sum_{j=1}^\infty \sqrt{\lambda_j} \, |\langle T e_j , e_j \rangle_H|.
\end{eqnarray*}
The conclusion now follows by observing that, for every $j\geq 1$, $|\langle T e_j , e_j \rangle_H|\leq \|T\|_{op}$, and by exploiting the relations
$$
\|T\|_{op} =\|\sqrt{S_1} - \sqrt{S_2}\|_{op} \leq \|{S_1} - {S_2}\|_{op}^{1/2} \leq \|{S_1} - {S_2}\|_{HS}^{1/2},
$$
where the first inequality in the previous display is a consequence of \cite[Theorem V.1.9 and Theorem X.1.1]{bhatia2013matrix}, and the second inequality is standard\footnote{The results from \cite{bhatia2013matrix} cited in our proof are stated in such a reference only in the case where $H$ is finite-dimensional (and, consequently, $S_1,\, S_2$ are matrices); the needed extension to a separable Hilbert space follows from a standard limiting procedure.}.
\end{proof}
}
We also record the following bound from \cite{bourguin2020approximation}: {for the sake of completeness, we provide here a direct proof neither appealing to the notion of abstract Wiener space nor assuming that $X_1$ and $X_2$ are non-degenerate (as in \cite[Corollary 3.3]{bourguin2020approximation}).}

\begin{prop}\label{p:bc} Let the assumptions and notation of the present section prevail. Then,
$$
d_2(X_1,X_2)\leq \frac12\| S_1-S_2\|_{HS}.
$$
\end{prop}
{\begin{proof}
Let $h\in C^2_b(H)$ be such that $\sup_{x\in H}\|\nabla^2h(x)\|_{H^{\otimes 2}}\leq 1$.
Without loss of generality, let us assume that $X_1$ and $X_2$ are independent, and let us set $U_t=\sqrt{t}X_1+\sqrt{1-t}X_2$ for $t\in[0,1]$. We have
\begin{eqnarray*}
    \mathbb{E}[h(X_1)] - \mathbb{E}[h(X_2)] &=& \int_0^1 \frac{d}{dt}\mathbb{E}[h(U_t)]dt\\
    &=&\int_0^1 \left( \frac{1}{2\sqrt{t}}\mathbb{E}\big[\langle \nabla h(U_t),X_1\rangle_H\big]
    - \frac{1}{2\sqrt{1-t}}\mathbb{E}\big[\langle \nabla h(U_t),X_2\rangle_H\big]\right)dt\\
    &=& \frac12 \int_0^1 \mathbb{E}\big[\langle \nabla^2 h(U_t),S_1-S_2\rangle_{HS}\big].
\end{eqnarray*}
Therefore
\begin{eqnarray*}
    \big|\mathbb{E}[h(X_1)] - \mathbb{E}[h(X_2)]\big| 
    &\leq & \frac12 \sup_{x\in H} \| \nabla^2 h(x)\|_{HS}\,\,\|S_1-S_2\|_{HS},
\end{eqnarray*}
and the desired conclusion follows.
\end{proof}
}
{{}
We will use Propositions \ref{p:bello} and \ref{p:bc} in combination with \eqref{e:w2cond}, in order to compare the distributions of $H$-valued random elements $Z,Y$ such that $Y$ is Gaussian and $Z$ is conditionally Gaussian. To simplify the discussion, the corresponding statement is provided below in the special case in which $H$ is a subspace of a $L^2$ space.

\begin{prop}\label{p:simple}Let $(T, \mathcal{T} , \nu)$ be a measure space such that $(T,\mathcal{T})$ is Polish and $\nu$ is a finite positive Borel measure. Write $L^2(\nu):= L^2(T, \mathcal{T} , \nu)$, consider a closed subspace $H_1\subset L^2(\nu)$, and select two $H_1$-valued random elements $Z,Y$ with the following properties:
\begin{itemize}

\item[--] $Y = \{Y(x) : x\in T\}$ is a centered Gaussian field with covariance $K(x,y) = \mathbb{E}[Y(x)Y(y)]$ such that $\int_T\int_T K(x,y)^2\nu(dx)\nu(dy), \int_T K(x,x)^2\nu(dx)<\infty$;
\smallskip
\item[--] there exists a symmetric positive definite random field $\Sigma = \{\Sigma(x,y) : x,y\in T\}$ such that
$\mathbb{E}\left[\int_T\int_T \Sigma(x,y)^2 \nu(dx)\nu(dy)\right],\,\, \mathbb{E}\left[\int_T \Sigma(x,x)^2 \nu(dx)\right] <\infty$ and, conditionally on $\Sigma$, $Z = \{Z(x) : x\in T\}$ is a centered Gaussian field with covariance $\Sigma$.



\smallskip

\end{itemize}
Then, the following estimates hold.
\begin{itemize}
    \item[\bf (1)] One has that
    $$
    d_2(Z,Y)\leq \frac12 \sqrt{\mathbb{E}\left[\int_T\int_T(K(x,y) - \Sigma(x,y))^2\,\nu(dx)\nu(dy) \right]}.
    $$
    \item[{\bf (2)}] Let $\lambda_1\geq \lambda_2\geq \cdots \geq 0$ denote the eigenvalues of the covariance $K$, that we identify with the integral operator $S_1 : H_1\to H_1$
    $$
    h\mapsto Sh := \int_T K(\cdot, y)h(y)\nu(dy).
    $$
    Also, denote by $S_2$ the (random integral operator) associated with the covariance $\Sigma$. 
    Then,
    \begin{equation}\label{e:traces}
    {\rm Tr}\, (S_1) = \int_T K(x,x)\nu(dx), \quad  {\rm Tr}\, (S_2) = \int_T \Sigma(x,x)\nu(dx) \quad \mbox{\rm (a.s.-$\mathbb{P}$)}, 
    \end{equation}
and
\begin{eqnarray}\label{e:choc}
{\bf W}_2(Z,Y)&\leq &\left\{\mathbb{E}\left[\int_T(K(x,x) - \Sigma(x,x))^2\,\nu(dx) \right]\right\}^{\frac14}\\
\notag && + 2^{\frac12} \left\{\mathbb{E}\left[\int_T\int_T(K(x,y) - \Sigma(x,y))^2\,\nu(dx)\nu(dy) \right]\right\}^{\frac18}.
\end{eqnarray}
At Point {\bf (1)} and Point {\bf (2)}, the distances $d_2$ and ${\bf W}_2$ are defined with respect to the Hilbert space $H_1$.
\end{itemize} 
    
\end{prop}

\begin{proof} To prove Point {\bf (1)}, we can assume without loss of generality that $Y, \Sigma$ are defined on the same probability space, and that $(Z,\Sigma)$ and $Y$ are stochastically independent. Now, for every $h\in C^2_b(H)$ such that  $\|h\|_{C^2_b(H)}\leq 1$ one has that 
$$
\Big| \mathbb{E}[h(Z)] -\mathbb{E}[ h(Y)]\Big| \leq \mathbb{E}\Big[\Big| \mathbb{E}[h(Z)\, |\, \Sigma] -\mathbb{E}[ h(Y) \, |\, \Sigma] \, \Big| \Big],
$$
and the result follows by applying Proposition \ref{p:bc} in the case $S_1 = \Sigma$ and $S_2 = K$. The proof of Point {\bf (2)} follows by applying \eqref{e:w2cond} to the case $q=2$ and $U=\Sigma$, and then by applying Proposition \ref{p:bello} to the case $S_1 = K$ and $S_2 = \Sigma$. Relation \eqref{e:traces} follows e.g. from the arguments rehearsed in \cite[Proof of Proposition 1.8]{daprato2006} and the fact that assumptions on $Z$ and $Y$ imply that $\mathbb{E}[\|Y\|_{H_1}^2], \, \mathbb{E}[\|Z\|_{H_1}^2] < \infty$.
\end{proof}
}

\section{Proof of the main results}\label{sec:pfs}

\subsection{Proof of Theorem \ref{thm:one-d}}\label{sec:one-d}

Fix $J$ and $x_\alpha$ as in the statement. Then, conditionally on $\mathcal{F}^{(L)}$, the random variable $V^J_\alpha z_{i;\alpha}^{(L+1)}$ is centered and Gaussian, with variance $V_\alpha^J V_\beta^J\Sigma^{(L)}_{\alpha\beta} \,|_{x_\alpha = x_\beta} := A$. Writing $d$ for either $d_{TV}$ or ${\bf W}_1$ and denoting by $Y$ a centered Gaussian random variable with variance $\mathbb{E}(A)$, we infer that
$$
d(V^J_\alpha z_{i;\alpha}^{(L+1)}, Z) \leq d(V^J_\alpha z_{i;\alpha}^{(L+1)}, Y) + d(Y,Z):= P + Q,
$$
and the conclusion of Point {\bf (1)} is obtained by bounding $P$ and $Q$ by means of \eqref{e:ktvbound}--\eqref{e:wbound} and \eqref{e:gauss1}--\eqref{e:gauss2}, respectively, and then by applying \eqref{eq:var-n} in the case $J_1=J_2 = J, \, \ell = L$ and $\alpha_1 = \alpha_2 = \alpha$. Point {\bf (2)} in the statement follows from \eqref{e:lowerbound} in the case $A = \Sigma^{(L)}_{\alpha\alpha}$ and $\sigma^2 = \mathbb{E}(\Sigma^{(L)}_{\alpha\alpha})$, that one should combine with \eqref{e:exact1}, and the fact that, in this specific configuration and by virtue of \eqref{e:uppercum}, 
\begin{equation}\label{e:uppity}
|R+ \mathbb{E}[(A-\sigma^2)^3] |\leq Q n^{-2},
\end{equation}
for some constant $Q$ independent of $n$. We observe that, in order to deduce \eqref{e:uppity}, we used the two elementary identities: $\mathbb{E}[(A-\sigma^2)^3] = \kappa_3(A)$, and $\mathbb{E}[(A-\sigma^2)^4] = \kappa_4(A) + 3 \kappa_2(A)^2$. 

\subsection{Proof of Theorem \ref{thm:finite-d}}\label{sec:finite-d} Write $M_0 := M\cdot n_{L+1}$. We start by observing that, conditionally on $\mathcal{F}^{(L)}$, the $M_0$-dimensional random vector $F:= \lr{V_{\alpha_\ell}^{J_\ell} z_{i;\alpha_\ell}}_{\substack{1\leq i \leq n_{L+1}\\  (J_\ell,\alpha_\ell) \in {\bf B}}}$ is Gaussian and centered, with covariance
$$
\Sigma(i,(J_\ell, \alpha_\ell)\,  ; \, j,(J_k, \alpha_k)  ) := \delta_{ij} V_{\alpha_\ell}^{J_\ell}V_{\alpha_k}^{J_k} \Sigma^{(L)}_{\alpha_\ell\alpha_k},
$$
where we used the convention \eqref{e:conv1} to deal with the case $\alpha_k = \alpha_\ell$. Gaussian integration by parts yields, in particular, that, for all twice differentiable functions $h : \mathbb{R}^{M_0} \to \mathbb{R}$ that are $1$-Lipschitz and such that
$$
\sup_{x\in \mathbb{R}^{M_0}}\| {\rm Hess}\, h(x)\|_{HS}  \leq 1,
$$
one has the identity
$$
\mathbb{E}[\langle \nabla h(F),  F\rangle_{\mathbb{R}^{M_0}}] = \mathbb{E}[\mathbb{E}[\langle \nabla h(F),  F\rangle_{\mathbb{R}^{M_0}}\, |\, \mathcal{F}^{(L)}]] = \mathbb{E}[\langle \Sigma , {\rm Hess}\, h(F)\rangle_{HS}].
$$
Now suppose that the assumptions of Point {\bf (1)} in the statement are in order. One can apply Proposition \ref{p:ibp2} in the case $M=M_0$ and $N = G$ to deduce that the quantity $d_c(F,G)$ is bounded by a multiple of $\sqrt{B_n}$, where $B_n$ is defined in \eqref{e:bn} with $\mu(dJ, dx)$ equal to the counting measure on ${\bf B}$, and the conclusion follows from \eqref{e:effective}. Similarly, under the assumptions of Point {\bf (2)} in the statement, one can exploit Proposition \ref{p:ibp3} in the case $M=M_0$ and $N = G'$ to deduce that the quantity $d_c(F, G')$ is bounded by a multiple of $\sqrt{A_n}$, where $A_n$ is defined in \eqref{e:an} with $\mu(dJ, dx)$ equal to the counting measure on ${\bf B}$, and \eqref{e:effective} yields once again the desired conclusion. \qed


\subsection{Proof of Theorem \ref{thm:infinite-d}}\label{sec:infinite-d} The statement follows from Proposition \ref{p:simple}, as applied to the following configuration
\begin{enumerate}
\item[--] $T = \mathcal{M}_q\times[n_{L+1}]\times \mathbb{U}$ and $\nu = \nu_0\otimes \nu_1 \otimes dx$, where $\nu_0$ and $\nu_1$ are counting measures;
\item[--] $Y = \Gamma^{(L+1)}_{\mathbb{U}}$, regarded as a random element with values in $H_1=\mathbb{W}^{q;2}(\mathbb{U})\subset L^2(\nu)$ ;
\item[--] $Z = z^{(L+1)}_{\mathbb{U}}$, regarded as a random element with values in $H_1=\mathbb{W}^{q;2}(\mathbb{U})\subset L^2(\nu)$ ;
\item[--] for $(J_1, i_1, x_{\alpha_1}), \, (J_2, i_2, x_{\alpha_2})\in T$, 
$$\Sigma((J_1,i_1,x_{\alpha_1}) ; (J_2,i_2,x_{\alpha_2}))  = \delta_{i_1 i_2} D^{J_1}_{\alpha_1}D^{J_2}_{\alpha_2} \Sigma^{(L)}_{\alpha_1\alpha_2}, $$
where the convention \eqref{e:conv1} has been implicitly applied.
\end{enumerate}
Proposition \ref{p:simple} implies therefore that, under the assumptions of Theorem \ref{thm:infinite-d}-{\bf (1)}, the quantity $d_2\left(z^{(L+1)}_{\mathbb{U}} , \Gamma^{(L+1)}_{\mathbb{U}}\right)$ is bounded by a multiple of $\sqrt{B_n}$, where $B_n$ is defined according to $\eqref{e:bn}$ in the case $\mu = \nu_0\otimes dx$, so that the conclusion follows from \eqref{e:effective}. Analogously, under the assumptions of Theorem \ref{thm:infinite-d}-{\bf (2)}, Proposition \ref{p:simple} yields that the quantity ${\bf W}_{2;q}\left(z^{(L+1)}_{\mathbb{U}} , \Gamma^{(L+1)}_{\mathbb{U}}\right)$ {{}is bounded by a multiple of $B_n^{\frac18} + C_n^{\frac14}$} (see \eqref{e:cn}), and \eqref{e:effective} yields once again the desired conclusion. The last statement in the theorem follows by an analogous route.

 \subsection{Proof of Theorem \ref{t:sup}} Fix $\mathbb{U}$ and $k\geq 1$ as in the statement, and define $r := k+1+\lfloor \frac{n_0}{2}\rfloor$. In view of \cite[Theorem 4.1]{villani2009transport}, it is sufficient to prove formula \eqref{e:winfi}. To accomplish this task, we will exploit relation \eqref{e:w2inficond} in the following setting: $X = z_{\mathbb{U}}^{(L+1)}$, $Y =\Gamma^{(L+1)}_{\mathbb{U}}$ and $V = \Sigma^{(L)}= \{\Sigma^{(L)}_{\alpha\beta}: x_\alpha, x_\beta \in \bar{\mathbb{U}}\}$, as defined in \eqref{eq:Sig-def}. We regard $z_{\mathbb{U}}^{(L+1)}$ and $\Gamma^{(L+1)}_{\mathbb{U}}$ as random elements with values in $C^k(\bar{\mathbb{U}})$, such that $\mathbb{P}(z_{\mathbb{U}}^{(L+1)} \in C^\infty(\bar{\mathbb{U}}) )= \mathbb{P}(\Gamma_{\mathbb{U}}^{(L+1)}\in C^\infty(\bar{\mathbb{U}})) = 1 $. Similarly, we regard   $\Sigma^{(L)}$ as a random element with values in the space $C^{k,k}(\bar{\mathbb{U}}\times \bar{\mathbb{U}})$ such that $\mathbb{P}_{\Sigma^{(L)}}(C^{\infty,\infty}(\bar{\mathbb{U}}\times \bar{\mathbb{U}}) )=1$, where $\mathbb{P}_{\Sigma^{(L)}}$ is shorthand for the law of $\Sigma(L)$. By construction, there exists a version of the conditional probability $$\mathbb{Q}_S := \mathbb{P}_{z_{\mathbb{U}}^{(L+1)}\,|\, \Sigma^{(L)} = S}$$ such that, for $\mathbb{P}_{\Sigma^{(L)}}$-almost every $S$, one has that $S\in C^{\infty,\infty}(\bar{\mathbb{U}}\times \bar{\mathbb{U}})$ and, under $\mathbb{Q}_S$, the random element $z_{\mathbb{U}}^{(L+1)}$ is a centered Gaussian random field with $n_{L+1}$ independent components with common covariance $S$; when these two requirements are met, one has that
 $$
 \mathbb{Q}_S(C^\infty(\bar{\mathbb{U}})) = 1.
$$
The following statement gathers together the main results one can deduce from the construction of coupled smooth Gaussian fields detailed in \cite[Section 4.2]{dierickx2023small}.

 \begin{lemma}[See \cite{dierickx2023small}]\label{l:1} Let the above notation and assumptions prevail, and let $S$ be a symmetric and positive definite element of $C^{\infty,\infty}(\bar{\mathbb{U}}\times \bar{\mathbb{U}})$. Let ${\bf K}$ be the operator defined in \eqref{e:hsop} for $r = k+1+\lfloor \frac{n_0}{2}\rfloor$, and let ${\bf K}_S$ be the operator obtained from \eqref{e:hsop} by replacing the kernel $K^{(L+1)}$ with $S$. Then, there exists a probability space $(\Omega_1, \mathcal{F}_1, \mathbb{P}_1)$ supporting two random elements $E,F,$ with values in $C^k(\bar{\mathbb{U}})$ and such that: 
 \begin{enumerate}
 \item[\rm (a)] $E$ has the law of a centered Gaussian field on $\bar{\mathbb{U}}$ with $n_{L+1}$ independent components having common covariance $S$;
 \item[\rm (b)] $F\stackrel{law}{=}\Gamma^{(L+1)}_{\mathbb{U}}$; 
\item[\rm (c)] the following estimate is in order: $$\mathbb{E}_1[\| E-F\|^2_{\mathbb{W}^{r;2}(\mathbb{U})}] = \| \sqrt{{\bf K}} -  \sqrt{{\bf K}_S}\|^2_{HS}.$$
 \end{enumerate}
     
 \end{lemma}

 \smallskip

 \noindent For $E,F$ as in Lemma \ref{l:1}, one has that $\mathbb{P}_1(E,F\in C^\infty(\bar{\mathbb{U}}))=1 $, and on can apply \eqref{e:sobolev} to deduce that, for some absolute constant $A$ depending on $\mathbb{U}$, one has the bound $\| E-F\|_{C^{k}(\bar{\mathbb{U}})}\leq A\cdot \| E-F\|_{\mathbb{W}^{r;2}(\mathbb{U})}$, a.s.-$\mathbb{P}_1$. Since, by virtue of Proposition \ref{p:bello}, one has that
 $$
  \| \sqrt{{\bf K}} -  \sqrt{{\bf K}_S}\|_{HS}\leq{\teal  c\cdot \Big(  \| {\bf K} -  {\bf K}_S\|^{\frac{1}{4}}_{HS} +|{\rm Tr}\,{\bf K}- {\rm Tr}\,{\bf K}_L  |^{\frac12} \Big) }
 $$
 for some finite constant $c$ uniquely depending on $p$ and on the deterministic operator ${\bf K}$, we deduce from \eqref{e:w2inficond} that ${\bf W}_{\infty; k}\left(z^{(L+1)}_{\mathbb{U}},\Gamma^{(L+1)}_{\mathbb{U}}\right)$ {{}is bounded by a multiple of $B_n^{\frac18} + C_n^{\frac14}$}, where $B_n, C_n$ are defined according to \eqref{e:bn} and \eqref{e:cn}, respectively, in the case $\mu = \nu_0\otimes dx$. The conclusion now follows from relation \eqref{e:effective}. \qed
\section{Statements and Declarations}
\subsection{Funding.} SF gratefully acknowledge support from the Italian Ministry of Education, University and Research, ``Dipartimenti di Eccellenza" grant 2023-2027. BH gratefully acknowledges support from NSF CAREER grant DMS-2143754 as well as NSF grants DMS-1855684, DMS-2133806 and an ONR MURI on Foundations of Deep Learning. DM is grateful to MUR projects \emph{MatModTov}, \emph{Grafia} and to PNRR/CN1 Spoke 3 for financial support. 
IN's research is supported by the Luxembourg National Research Fund (Grant: O22/17372844/FraMStA).
GP's research is supported by the Luxembourg National Research Fund (Grant: O21/16236290/HDSA).

\subsection{Other Interests.} The authors declare no financial or non-financial competing interests. 

\subsection{Acknowledgement.} We thank Nicholas Nelsen for pointing out a mistake in the previous version of this paper. We also thank two anonymous referees for pointing out a range of errors in an earlier version as well as for suggesting ways to improve the exposition in a number of places.

 \bibliography{bibliography}
\bibliographystyle{plain}
\end{document}